\documentclass[11pt]{article}
\usepackage{amssymb}
\usepackage{amsmath}
\usepackage{amsthm}
\usepackage{url} 
\usepackage{natbib}
\usepackage[normalem]{ulem}
\usepackage{hyperref}

\usepackage{cleveref}

\usepackage{float}
\usepackage{lipsum}
\usepackage[ruled]{algorithm2e}
\usepackage{geometry}
\usepackage[title]{appendix}

\usepackage{booktabs}
\usepackage{soul}
\usepackage[authormarkup=none,todonotes={textsize=scriptsize}
]{changes}
\usepackage{subcaption} 
\usepackage{multirow}
\usepackage{xr}
\usepackage{tablefootnote}
\usepackage{pifont}
\usepackage{algorithmic}
\newtheorem{theorem}{Theorem}[section]

\newtheorem{lemma}[theorem]{Lemma}
\newtheorem{corollary}[theorem]{Corollary}
\newtheorem{definition}[theorem]{Definition}
\newtheorem{assumption}[theorem]{Assumption}

\input{macros.tex}

\newcommand\tnorm[1]{\left\vert\xspace\left\vert\xspace\left\vert\mskip2mu
#1\mskip2mu \right\vert\xspace\right\vert\xspace\right\vert}

\newcommand{\keywords}[1]
{
  \small	
  \textbf{\textit{Keywords---}} #1
}

\setdeletedmarkup{\IfIsColored{\color{authorcolor!20}}((#1)}

\definechangesauthor[color=blue]{KL}
\definechangesauthor[color=purple]{SH}

\author{Haitham Kanj, and Kiryung Lee\thanks{The authors are with the Department of Electrical and Computer Engineering at the Ohio State University (corresponding author: kanj.7@osu.edu). This work was supported in part by NSF CAREER Award CCF-1943201 and is currently submitted for review in the Journal of Information and Inference. } \\
Department of Electrical and Computer Engineering\\
The Ohio State University, Columbus, OH, USA}

\begin{document}

%\subtitle{Subject Section}

%\revised{Date}{0}{Year}
%\accepted{Date}{0}{Year}

%\editor{Associate Editor: Name}

\title{Locally Near Optimal Piecewise Linear Regression in High Dimensions via Difference of Max-Affine Functions}
\maketitle
%%%%%%%%%%%%%%%%%%%% ABSTRACT
\abstract{
This paper presents a parametric solution to piecewise linear regression through the Adaptive Block Gradient Descent (ABGD) algorithm. 
The heart of the method is the parametrization of piecewise linear functions as the difference of max-affine (DoMA) functions. 
A non-asymptotic local convergence analysis for ABGD is provided under sub-Gaussian covariate and noise distributions. To initialize ABGD, we adapt a prior algorithm originally developed for the simpler setting of max-affine functions.
When suitably initialized, ABGD converges linearly to an $\epsilon$-accurate estimate given $\tilde{\mathcal{O}}(d\max(\sigma_z/\epsilon,1)^2)$ observations where $\sigma_z^2$ denotes the noise variance. This implies exact recovery given $\tilde{\mathcal{O}}(d)$ samples in the noiseless case. Also, such a rate is shown to be minimax optimal up to logarithmic factors.  Synthetic numerical results corroborate the theoretical guarantees for ABGD. We also observe competitive performance compared to the state-of-the-art methods on real-world datasets.
}

%%%%%%%%%%%%%%%%%%%%%%%%%%%%%%%%%%%%%%%%%%%%
\section{Introduction}\label{Sec:Intro}
Nonlinear regression has been widely used in various applications from machine learning and statistics \citep{gallant1975nonlinear,bates1988nonlinear}. 
Established solutions to nonlinear regression can be categorized as non-parametric, semi-parametric, or fully parametric \citep{hardle1993comparing,mahmoud2019parametric}. Non-parametric regression assumes no prior knowledge of the nonlinear function that maps the covariates to the target variable. However, in the high-dimensional setting, this methodology suffers from significant estimation variance due to the scarcity of the data. On the other hand, fully parametric regression utilizes prior knowledge of the nonlinear function up to a few unknown model parameters, thus enabling low-cost and sample-efficient algorithms. However, a model mismatch can lead to a large estimation bias. Alternatively, semi-parametric regression introduces a flexible model for the covariate-target relation given as the composition of linear dimensionality reduction followed by a nonlinear function. 
Semi-parametric regression serves as a middle ground between the two previous methodologies, by balancing the strengths and weaknesses. 
%In this case, the link function is defined as a mapping from a few linear combinations of the covariates to the target variable. 
%Such models suffer less estimation bias compared to fully parametric models and admit better scaling in high dimensions compared to non-parametric models. 
Ultimately, the choice of methodology should be guided by a researcher's judgment according to their specific application. 

Regardless of the choice of methodology, limiting the class of nonlinear functions with respect to specific properties (e.g. continuity or smoothness) is necessary for the feasibility of regression. 
%Regardless of the choice of methodology, we must limit the regression problem to a class of functions determined by a set of conditions (usually on continuity or differential orders). 
In this paper, we consider the class of continuous piecewise-linear (PL) functions.
A useful property is that any function in this
class can be written as the difference of two max-affine (convex piecewise linear) functions \citep[Proposition 1]{siahkamari2020piecewise}. This parametrization is at the heart of the proposed method.
PL functions enjoy interpretability by inherently partitioning the covariate domain and identifying the weight of each covariate in explaining the target variable. 
The parametric approach proposed in this paper limits the number of affine functions in the PL fit, thereby improving interpretability.
Furthermore, it has been shown that PL functions can uniformly approximate any continuous nonlinear function with bounded gradients \citep[Proposition 2.2]{bavcak2011difference}.

\renewcommand{\arraystretch}{1}
\begin{table*}[t]
\centering
\caption{Comparison of nonlinear regression methodologies over different function classes with associated non-asymptotic sample size guarantees and runtimes.}
\label{Table:reg_ref}
\resizebox{\textwidth}{!}{
\begin{tabular}{l l c c c}
\toprule
\textbf{Method} & \textbf{Modeling Assumption} & \textbf{Sample Size} & \textbf{Runtime} & \textbf{Reference} \\ 
\midrule

Difference of Convex 
& Continuous with bounded gradients 
& $\mathcal{O}(e^d)$ 
& $\mathcal{O}(n^5 d^2)$ 
& \cite{siahkamari2020piecewise} \\ 

\addlinespace

Adaptive Regression Splines  
& $M$ tensor-product basis functions  
& $\infty$ 
& $\mathcal{O}(n d M^2)$ 
& \cite{eckle2019comparison} \\ 

\addlinespace

Nearest Neighbors Regression  
& $k$-nearest piecewise constant 
& $\mathcal{O}(e^d)$  
& $\mathcal{O}(n d \log k)$ 
& \cite{kpotufe2011} \\ 

\addlinespace

Multi-Index Learning    
& Depends on a covariate subspace 
& $\mathrm{poly}(d)$   
& $\mathrm{poly}(d)$ 
& \cite{bietti2023learning} \\ 

\addlinespace

ReLU Deep Neural Network    
& $L$-layer piecewise linear model 
& $\infty$   
& $\mathcal{O}(n d L)$ 
& \cite{li2017convergence} \\ 

\addlinespace

Difference of Max-Affine 
& $k$-segment piecewise linear 
& $\mathbf{\mathcal{O}(d)}$  
& $\mathcal{O}(n d^2 e^k)$
& \textbf{This paper} \\ 

\bottomrule
\end{tabular}}

\end{table*}

\subsection{Connection to Tropical Algebra} PL functions can be viewed as objects from max-plus algebra, a special case of tropical algebra \citep{maragos2021tropical,zhang2018tropical}. 
Specifically, max-affine and PL functions are referred to as tropical polynomials and tropical rational functions, respectively. 
Results from tropical algebra show that any deep neural network (DNN) with PL activations (e.g. rectified linear unit, or absolute value) can be rewritten as a difference of two max-affine functions. 
For example, consider a PL function written as a difference of max-affine functions with $k_1$ and $k_2$ affine segments, respectively.
Then it has been shown that this PL function can be represented as an $L$-layer ReLU DNN where $L \leq \lceil\log_2\max(k_1,k_2)\rceil +2$ \citep[Theorem 5.4]{zhang2012sparse}. Existing theoretical results from tropical algebra provide elegant geometric interpretations for PL regression, which will be elaborated in Section \ref{Sec:Tropical}. 
\subsection{Relevant Work} We highlight relevant theoretical results on nonlinear regression. 
\citet{siahkamari2020piecewise} proposes a non-parametric solution over the class of difference of convex functions.
This method is agnostic to any knowledge about the nonlinear mapping within the class of difference of convex functions with bounded gradients. 
However, their method inherits the sample and computation requirements of non-parametric methods. Specifically, the sample complexity $n$ scales exponentially in the ambient dimension $d$. Furthermore, their worst-case computational cost of $\mathcal{O}(n^5d^2)$ becomes infeasible for large datasets.
More recently, \citet{bietti2023learning} proposed a semi-parametric solution to nonlinear regression. 
They assumed that the covariate-target relation is the composition of linear dimensionality reduction followed by an unknown nonlinear function. 
The authors study gradient flow under the following setup: (i) the covariates follow the multivariate Gaussian distribution, (ii) the target is noise-free, (iii) the number of samples is infinitely large (population level), and (iv) the step size is infinitesimal (continuous time). 
Under these conditions, \citet[Theorem 3.25]{bietti2023learning} provided global convergence guarantees of gradient flow. 
However, regarding the sample complexity, they only provided a conjecture that it would scale as $\mathrm{poly}(d)$. 

\subsection{Contributions}
This paper presents theoretical guarantees for the \textit{adaptive block gradient descent} (ABGD) algorithm for parametric PL regression. 
Since the adaptive step size formula is explicitly given for ABGD, no hyperparameter tuning is required.
Furthermore, ABGD enjoys the computational cost of gradient-based methods as $\mathcal{O}(nd)$ per iteration. 
The heart of ABGD is the parametrization of PL functions as the difference of two max-affine functions (DoMA), i.e. $\bm x \mapsto \max_{j}\langle\bm \beta_j, [\bm x;1]\rangle - \max_{l}\langle\bm \alpha_l, [\bm x;1]\rangle$. 
We provide a non-asymptotic local convergence analysis for ABGD under sub-Gaussian noise when the covariate distribution satisfies sub-Gaussianity and anti-concentration properties. 
We consider the scenario where no pair of $( \bm \beta_j, \bm \alpha_l)$ for all $j\in \{1,\ldots,k_1\}$ and $l\in \{1,\ldots,k_2\}$ achieves the two maxima in the DoMA function simultaneously with an overwhelming probability. 
In this case, a suitably initialized ABGD converges linearly to an $\epsilon$-accurate estimate given $\mathcal{O}(d\max(\epsilon^{-2}\sigma_z^2,1))$ observations where $\sigma_z^2$ denotes the noise variance, thereby achieving optimal linear dependence on $d$ up to logarithmic factors. A comparison of existing methods for nonlinear regression is provided in Table \ref{Table:reg_ref}. Our method is the first to balance to achieve a linear dependence of the sample size $n$ on the ambient dimension at the cost of exponential runtime scaling in the model order $k$. This runtime becomes quite reasonable when $k \ll \max(n,d)$, which we assume.\footnote{The runtime of our method reported in Table \ref{Table:reg_ref} as $\mathcal{O}(nd^2e^k)$ assumes the randomized implementation of the PCA step for the initialization algorithm, which allows the quadratic dependence on the ambient dimension $d$. Also, the exponential dependence on $k$ is due to a brute force net search, which could be replaced by a randomized search with $\mathcal{O}(k)$ runtime in practice. } 
The size of the basin for linear convergence depends on the geometric properties of the ground-truth parameters. 
%Since the corresponding parameter estimation casts as a non-convex optimization problem, it is important for ABGD to start from a suitable initialization within the basin of attraction of the ground truth. 
In Section~\ref{Sec:Numerical}, we corroborate these theoretical guarantees and compare ABGD to relevant methods on real-world datasets. 
Compared to the semi-parametric and non-parametric methods described earlier, the parametric PL regression by ABGD provides better sample complexity scaling and lower computational costs. 
Furthermore, we investigate methods to obtain initial estimates of the model parameters. 
We also present that the spectral method developed for max-affine can be extended to DoMA regression. 
In particular, we characterize the relation between subspaces spanned respectively by the pairwise differences and all parameter vectors. 

\subsection{Notation} We use lightface characters to denote scalars, lowercase boldface characters to denote column vectors, and uppercase boldface to denote matrices. 
We also adopt the symbols for the max and min operators in lattice theory, i.e. $a \vee b = \max(a,b)$ and $a \wedge b = \min(a,b)$ for $a,b \in \mathbb{R}$. We use multiple matrix norms.
The Frobenius norm, the spectral norm, and the largest magnitude of entries will be respectively denoted by $\|\cdot\|_\mathrm{F}$, $\|\cdot\|$, and $\|\cdot\|_\infty$. We also reuse $\|\cdot\|$ for the Euclidean norm. 
We use a shorthand notation $[d]$ for the set $\{1,2,\dots,d\}$.
For a column vector $\bm x \in \mathbb{R}^d$, its sub-vector with the entries indexed by $\mathcal{S} \subset [d]$ is denoted by $[\bm x]_\mathcal{S}$. 
Similarly, for a matrix $\bm X \in \mathbb R^{d\times d}$, its sub-matrix with the entries indexed by $\mathcal{S}_1 \times \mathcal{S}_2 \subset [d]\times [d]$ is denoted by $[\bm X]_{\mathcal{S}_1,\mathcal{S}_2}$. Finally, we denote by $C, C_1, C_2, \dots$ universally absolute constants, not necessarily the same at each occurrence.
\begin{figure*}[ht]
        \centering
        \includegraphics[width=0.9\textwidth]{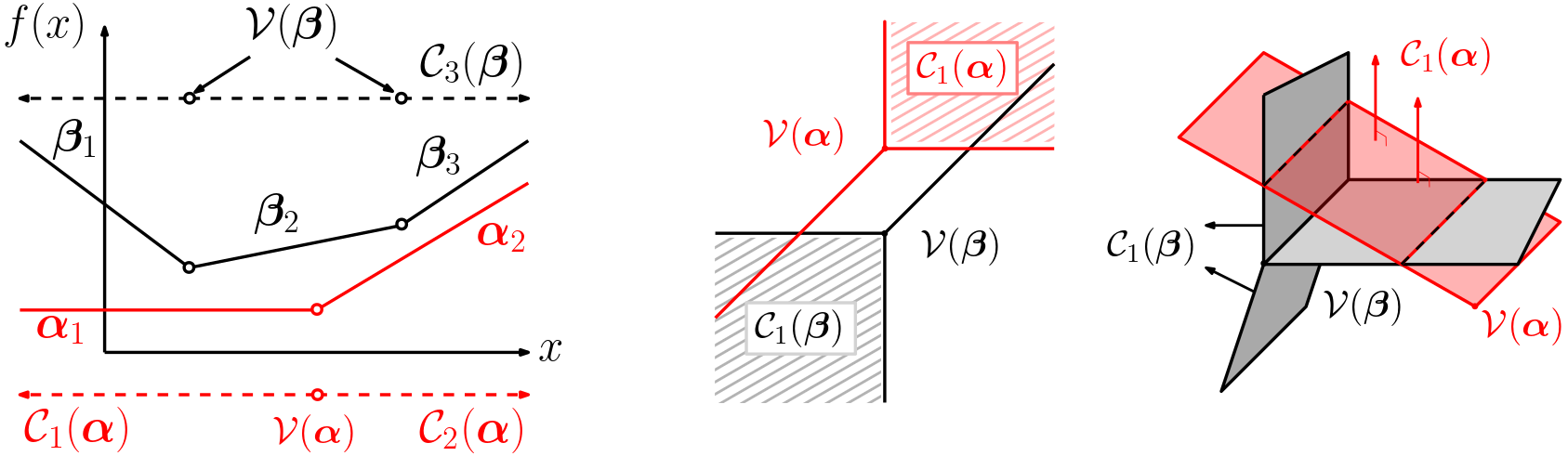} 
\caption{Examples of DoMA functions with $(\bm \beta, \bm \alpha)$ partitioning $\mathbb R$ (left), $\mathbb R^2  $ (middle) and $\mathbb R^3$ (right) with $\mathcal{V}$ denoting the boundary set, and $\mathcal{C}_i$ denoting the $i$th affine function activation set. 
}
\label{fig:Partion_sets}
\end{figure*}

\section{Problem statement}
%Piecewise Linear Regression}
We parametrize the piecewise linear function $f: \mathbb{R}^d \mapsto\mathbb R$ for the covariate-target relation as
\begin{align} \label{eq:classconf}
f(\bm  x;\bm  \beta, \bm  \alpha) = \max_{j \in [k_1]} \langle \bm  \beta_j, [\bm  x ;1]\rangle  - \max_{l \in [k_2]} \langle \bm  \alpha_l, [\bm  x ;1]\rangle, 
\end{align}
where $\bm \beta := [\bm {\beta}_1;\ldots;\bm {\beta}_{k_1}]$ and $\bm \alpha := [\bm {\alpha}_1;\ldots;\bm {\alpha}_{k_2}]$ collect the parameters for the corresponding affine functions. 
The resulting piecewise linear regression model is written as
\begin{equation} \label{eq:dataset}
     y_i = f(\bm  x_i; \bm  \beta^\star, \bm  \alpha^\star) + z_i,  \quad \forall i= 1, \ldots, n,
\end{equation}
where $\bm  \beta^\star$ and $\bm  \alpha^\star$ are the ground-truth parameters, $\{\bm  x_i\}_{i=1}^n$ (resp. $\{y_i\}_{i=1}^n$) are the samples of covariates (resp. target), and $\{z_i\}_{i=1}^n$ denote additive noise. 
To simplify notation, we introduce $\bm{\xi}_i := [\bm  x_i;1] $ for all $i \in [n]$. 
Then the regression model in \eqref{eq:dataset} is rewritten as a difference of max-linear functions model given by
\begin{equation}
\label{eq:def_model_xi}
y_i = \max_{j \in [k_1]} \langle \bm \xi_i,\bm{\beta}_j^\star \rangle - \max_{l \in [k_2]} \langle \bm \xi_i,\bm{\alpha}_l^\star \rangle + z_i.
\end{equation}

We consider the least squares estimator that minimizes the Mean Squared Error (MSE) loss function
\begin{equation}
\label{eq:loss_def_max0}
\ell\left( \bm \beta , \bm \alpha\right) =
\frac{1}{2n}\sum_{i=1}^{n}\left(y_i- \max_{j \in [k_1]} \langle\bm \xi_i,\bm \beta_j\rangle +  \max_{l \in [k_2]} \langle\bm \xi_i,\bm \alpha_l\rangle \right)^2.
\end{equation}
Since the loss function in \eqref{eq:loss_def_max0} is non-convex and non-smooth, the parameter estimation requires a careful design of the optimization algorithm. 
We propose a variant of the block gradient descent algorithm and present its convergence analysis in the following sections.

\section{Adaptive Block Gradient Descent Algorithm}

\begin{algorithm}[tb]
   \caption{(Adaptive Block Gradient Descent)}
   \label{algo:AlternatingGD}
\begin{algorithmic}
   \STATE {\bfseries Input:}  Dataset $\{\bm x_i, y_i\}_{i=1}^{n} $, model orders $(k_1,k_2)$, initial estimates $\bm \beta^0 = [\bm \beta_1^0; \ldots; \bm \beta_{k_1}^0] $, $\bm \alpha^0= [\bm \alpha_1^0; \ldots; \bm \alpha_{k_2}^0]$, and stop criterion $\gamma $. 
   \vspace{0.1cm}
   \REPEAT
   \vspace{0.2cm}
   \STATE \textbf{1. }${\bm\beta}^{t+1}_j \gets  \bm{\beta}_j^{t} - {\mu}_j(\bm \beta^t)  \nabla_{\bm \beta_j} \bm{\ell} \left( \bm{\beta}^{t}, \bm{\alpha}^{t}\right) $   $\quad \hspace{0.05cm}~~~\forall j \in [k_1]$
   \vspace{0.2cm}
   \STATE \textbf{2. }${\bm\alpha}^{t+1}_l \gets  \bm{\alpha}_l^{t} - {\mu}_l(\bm \alpha^t)  \nabla_{\bm \alpha_l} \bm{\ell} \left( \bm{\beta}^{t+1}, \bm{\alpha}^{t}\right) $ $~~~ \forall l \in [k_2]$
   \vspace{0.05cm}
   \STATE \textbf{3. } ${t\gets t+1}$
   \vspace{0.2cm}
   \UNTIL{$\displaystyle \max \left({\|\bm \beta^{t+1}-\bm \beta^t\|_2}/{\|\bm \beta^t\|_2},{\|\bm \alpha^{t+1} -\bm \alpha^{t}\|_2}/{\|\bm \alpha^t\|_2}\right)\leq \gamma$}
   \vspace{0.1cm}
   \STATE {\bfseries Output:} Final estimates $\hat{\bm \beta}= [\bm \beta_1^{t}; \ldots; \bm \beta_{k_1}^{t}]$, and $\hat{\bm \alpha} =[\bm \alpha_1^{t}; \ldots; \bm \alpha_{k_2}^{t}]$ 
\end{algorithmic}
\end{algorithm}

The adaptive block gradient descent (ABGD) algorithm to minimize the loss function in \eqref{eq:loss_def_max0} is summarized in Algorithm \ref{algo:AlternatingGD}.
ABGD is a variant of the block gradient descent algorithm, where the gradient is substituted by the generalized gradient \citep{hiriart1979new} and the step size varies across blocks adaptively with the iterates. 

To describe the algorithm, we introduce geometric objects derived from a max-affine function $\bm x \mapsto \max_{j \in [k]} \langle\bm \theta_j, [\bm x;1]\rangle$ of order $k$. 
On one hand, the $j$th affine function $\langle\bm \theta_j, [\bm x;1]\rangle$ achieves the maximum on the set 
\begin{equation}\label{eq:def_caljstar}
   \mathcal{C}_j(\bm \theta) := \{ \bm x \in \mathbb{R}^{d} ~:~ \langle [\bm x ;\ 1], \bm \theta_{j} - \bm \theta_{j'} \rangle > 0, ~ \forall j' \neq j \}
\end{equation}
for $j \in [k]$. 
On the other hand, multiple affine functions achieve the maximum simultaneously on the set 
\[
\mathcal{V}(\bm \theta) := \bigcup_{j' \neq j} \{ \bm x \in \mathbb{R}^{d} ~:~ \langle [\bm x ;\ 1], \bm \theta_{j}-\bm \theta_{j'} \rangle = 0 \},
\]
which is a null set with respect to any absolutely continuous probability measure. 
Then the embedding space $\mathbb{R}^{d+1}$ is partitioned into the maximizer sets $\{\mathcal{C}_j(\bm \theta)\}_{j=1}^k$ and their boundaries in $\mathcal{V}(\bm \theta)$. 
Indeed, $\{\mathcal{C}_j(\bm \theta)\}_{j=1}^{k}$ and $\mathcal{V}(\bm \theta)$ respectively denote open cells and zero-set in a special instance of tropical algebra called the max-plus algebra \citep{maragos2021tropical}. Figure \ref{fig:Partion_sets} presents examples of these sets. Furthermore, we can define the collection of training sample indices that activate the affine model parameter $\bm \beta_j$ as
\begin{equation}
    \mathcal{I}_j(\bm \beta) = \left\{i \in [n]: \bm x_i \in \mathcal{C}_j(\bm \beta)  \right\}. 
\end{equation}
Each iteration of ABGD starts by updating the estimate of $\bm \beta = [\bm {\beta}_1;\ldots;\bm {\beta}_{k_1}]$ via the block-wise gradient step 
\[
{\bm\beta}^{t+1}_j = \bm{\beta}_j^{t} - {\mu}_j^t(\bm \beta^t)  \nabla_{\bm \beta_j} \bm{\ell} \left( \bm{\beta}^{t}, \bm{\alpha}^{t}\right), 
\]
where the partial generalized gradient with respect to the variables in the $j$th block $\bm \beta_j$ is written as
\begin{equation}
\label{p_gradient} 
    \nabla_{\bm \beta_j}\ell(\bm \beta, \bm \alpha) 
    = \frac{1}{n} \sum_{i\in \mathcal{I}_j(\bm \beta)} \left(\langle\bm \xi_i,\bm \beta_j\rangle - \max_{l \in [k_2]} \langle \bm \xi_i,\bm{\alpha}_l \rangle-y_i\right)\bm \xi_i.
\end{equation}
This is shown to be a valid generalized gradient by \citet[Equation~7]{kim2024max}. Using $|\mathcal{I}_j(\bm \beta^t)|$ to denote the number of samples within the $j$th maximizer set, the step size used for updating the estimator for the $j$th block of $\bm{\beta}^t$ is determined adaptively as
\begin{equation}\label{eq:def_mu}
\mu_j(\bm{\beta}^{t}) = \begin{cases}
    n/|\mathcal{I}_j(\bm \beta^t)|%\left(\frac{n_j(\bm \beta^t)}{n} \right)^{-1} 
    & \text{if } ~|\mathcal{I}_j(\bm \beta^t)| \geq 1 \\ 
    0 &  \text{otherwise}.
\end{cases}  
\end{equation}
This step-size corresponds to the inverse of the empirical probability measure of the $j$th maximizer set. Note that if $|\mathcal{I}_j(\bm \beta^t)| =0$ then the affine model parameterized by $\bm \beta_j^t$ is inactive and consequently not modified by a gradient step.
Once the estimate for $\bm{\beta}$ is completed, the estimator for $\bm{\alpha}$ is updated similarly with $\bm{\beta}^{t+1}$ via the block-wise gradient step where 
the partial generalized gradient with respect to the $l$th block $\bm \alpha_l$ is written as
\begin{equation*}
     \nabla_{\bm \alpha_l}\ell(\bm \beta, \bm \alpha) 
    = \frac{1}{n} \sum_{i\in \mathcal{I}_l(\bm \alpha)} \left(\langle\bm \xi_i,\bm \alpha_l\rangle - \max_{j \in [k_1]} \langle \bm \xi_i,\bm{\beta}_j \rangle+y_i\right)\bm \xi_i
\end{equation*}
and the step size is determined by \eqref{eq:def_mu} with $\bm{\beta}^t$ substituted by $\bm{\alpha}^t$. 
This update rule applies recursively until the algorithm converges by satisfying the stop condition for a small absolute constant $\gamma > 0$.

\section{Theoretical Analysis of ABGD}\label{sec:theo}
We present the local convergence analysis of ABGD under a general class of covariate distributions that satisfy the following pair of properties. 
\begin{assumption}[Sub-Gaussianity]\label{assum:subg}
    The covariate vector $\bm x \in \mathbb{R}^d$ is zero-mean, isotropic, and $\eta$-sub-Gaussian, i.e. there exists $\eta >0$ such that %the following STATEment holds deterministically 
    \begin{equation}
        \sup\limits_{\bm u \in \mathbb{S}^{d-1}}\mathbb E \left[\mathrm{exp}\left(t\langle \bm u , \bm x\rangle\right)\right]  \leq \exp\left({t^2{\eta^2}/{2}}\right),\quad \forall~ t\in \mathbb R, \nonumber
    \end{equation}
    where $\|\cdot \|_{\psi_2}$ and $\mathbb{S}^{d-1}$ denote the sub-Gaussian norm and the unit sphere in $\ell_2^d$, respectively. %, and $\eta > 0 $ is some numerical constant.
\end{assumption} 
\begin{assumption}[Anti-Concentration]\label{assum:anti}
    There exist $\gamma,\zeta >0$ such that the covariate vector $\bm x \in \mathbb{R}^d$ satisfies
    \[
    \sup\limits_{\begin{subarray}{c} \bm u \in \mathbb{S}^{d-1}, \lambda \in \mathbb{R} \end{subarray}} 
    \mathbb{P}\left[ \left( \langle\bm u, \bm x \rangle +\lambda\right)^2\leq \epsilon \right]\leq \left( \gamma \epsilon\right)^\zeta, \quad \forall \epsilon >0.
    \]
\end{assumption}
On one hand, the Sub-Gaussian family is simply the collection of all densities that eventually (i.e., for any $|x| >a>0$) admit a decay rate equal to or faster than $\exp (-x^2)$. On the other hand, the Anti-Concentration property was shown to be satisfied by any log-concave distribution \citep[Appendix G]{ghosh2019max}. A non-exhaustive list of distributions that satisfy both properties is: (i) Gaussian (or Gaussian mixture), (ii) Uniform distribution on $[a,b]$, $\ell_2$-sphere or any bounded set, (iii) Beta distribution, $\mathrm{Beta}(\alpha,\beta)$ with $\alpha = \beta$, (iv) any distribution with bounded support (e.g. exponential, Laplace, Cauchy, etc), and (v) any mixture of these densities. Therefore, the scope of distributions we consider extends from the typical Gaussian assumption to include more non-trivial multi-modal distributions that can be encountered in machine learning applications and statistics.

Next, the statement of the main theorem makes use of the following definitions. First, we define a class of parameters that satisfy some geometric properties. 
\begin{definition} \label{eq:Theta_class}
Let $\kappa > 0$ and $\pi_{\min}, r \in [0,1]$. 
We define $\Theta(\kappa, \pi_{\min},r)$ to be a collection of all $[\bm \beta; \bm \alpha]$ with $\bm \beta =[\bm \beta_1; \ldots; \bm \beta_{k_1}]\in \mathbb{R}^{k_1(d+1)}$, $\bm \alpha =[\bm \alpha_1; \ldots; \bm \alpha_{k_2}]\in \mathbb{R}^{k_2(d+1)}$ satisfying the following properties: First, distinct weight vectors are separated by at least a discrepancy value $\kappa$, i.e.
\begin{equation}
\label{eq:defkappa}
\min_{j'\neq j}\|[\bm \beta_j-\bm \beta_{j'}]_{1:d} \|_2 \wedge \min_{l'\neq l}\|[\bm \alpha_l-\bm \alpha_{l'}]_{1:d} \|_2
\geq \kappa.
\end{equation}
Second, the probability that any linear function achieves the maximum in \eqref{eq:def_model_xi} with random $\bm x$ is at least $\pi_{\min}$, i.e.
\begin{equation}
\label{eq:def_pimin_pimax}
\min_{j\in[k_1]}\mathbb{P}\left(\bm x\in \mathcal{C}_j(\bm \beta) \right)\wedge\min_{l\in[k_2]}\mathbb{P}\left(\bm x\in \mathcal{C}_l(\bm \alpha) \right) \geq \pi_{\min}.
\end{equation}
Third, the probability that $(\bm \beta_j,\bm \alpha_l)$ achieves the two maxima in \eqref{eq:classconf} simultaneously with random $\bm x$ is bounded as 
    \begin{align}
        \label{eq:proba_intersection}
        \mathbb{P}&\left( \bm x \in \mathcal{C}_j(\bm \beta)\cap {C}_l(\bm \alpha)\right)\nonumber \leq \frac{\mathbb{P}\left( \bm x \in \mathcal{C}_j(\bm \beta)\right) }{k_1^r}\wedge\frac{\mathbb{P}\left( \bm x \in \mathcal{C}_l(\bm \alpha)\right) }{k_2^r},
    \end{align}
    for every affine pair $ (j,l) \in [k_1] \times [k_2]$. We can now proceed with the statement of the main theorem.
\end{definition}

\begin{theorem}\label{Theo:main}
Suppose that $\{(\bm x_i,z_i)\}_{i=1}^n$ are independent copies of a random tuple $(\bm{x},z)$ where $\bm{x} \in \mathbb R^d$ and $z \in \mathbb{R}$ are independent, $\bm{x}$ satisfies Assumptions \ref{assum:subg}--\ref{assum:anti} with parameters $\eta,\gamma,\zeta > 0$, and $z$ is zero-mean sub-Gaussian with variance $\sigma_z^2$. 
Then there exist absolute constants $C_1,C_2,c_3,R > 0$, for which the following statement holds with probability at least $1- ke^{-d}$ for all $[\bm \beta^\star;\bm \alpha^\star] \in \Theta(\kappa, \pi_{\min},r)$ with $r\in (1/2,1]$ and $k = k_1 \vee k_2$.
If $k> c_3$ and the initial estimates $(\bm \beta^0,\bm \alpha^0) $ satisfy
\begin{equation}\label{eq:nbr_theorem}
    \|\bm \beta^0-\bm \beta^\star\|_2 \vee \|\bm \alpha^0 - \bm\alpha^\star\|_2 \leq \kappa \rho
\end{equation}
with
\begin{equation}
\label{eq:choice_rho1}
\rho:= \left[\frac{R\pi_{\min}^{3/4}}{4 k^{2}} \cdot \log^{-1/2}\left(\frac{ k^{2}}{R\pi_{\min}^{3/4}}\right)\right] \wedge \frac{1}{4}
\end{equation}
and
\begin{equation}
\label{eq:cond:lem:lwb_gradient}
n \geq C_1 d\left(\sigma_z^2\vee 1\vee \eta^4 \right)  k^4\pi_{\min}^{-4(1+\zeta^{-1})}, 
\end{equation} 
then the sequence $\left(\bm \beta^t, \bm \alpha^t\right)_{t\in\mathbb{N}}$ by ABGD satisfies 
\begin{equation} \label{eq:theo_param_error}
\begin{aligned}
 \left\|\begin{bmatrix} \bm \beta^{t} \\ \bm \alpha^{t} \end{bmatrix} - \begin{bmatrix} \bm \beta^\star \\ \bm \alpha^\star \end{bmatrix} \right\|^2_2 \leq& \tau^t\left\|\begin{bmatrix} \bm \beta^{0} \\ \bm \alpha^{0} \end{bmatrix} - \begin{bmatrix} \bm \beta^\star \\ \bm \alpha^\star \end{bmatrix}\right\|^2_2 
 +\frac{C_2\sigma_z^2 d k \log(n/d)}{n}
\end{aligned}
\end{equation}
for $\tau\in (0,1)$ determined by $r$, $\pi_{\min}$, $k$, $\gamma$, $\zeta$ and $R$. 
\end{theorem}
The proof is deferred to Appendix \ref{append:proof:theo}. Theorem \ref{Theo:main} implies local linear convergence of a properly initialized ABGD under the assumed covariate model. The speed of convergence, determined by $\tau \in (0,1)$, is provided explicitly in \eqref{eq:bjt_bound_final}. We next present three key interpretations of Theorem \ref{Theo:main}.
\begin{enumerate}
    \item The minimax optimal linear scaling of the sample size $n$ with respect to $d$ is achieved up to the logarithmic factor $\log(n/d)$ (see Lemma \ref{lemm:minimax} in the Appendix for the minimax lower bound).
    \item In the ``well-balanced'' case, i.e. $\pi_{\min}=\Omega(1/k)$, the size of the basin of attraction $\rho$ in \eqref{eq:choice_rho1} becomes $\Omega(k^{-11/4})$. Therefore, the basin of attraction shrinks only by the orders $k_1,k_2$ of the DoMA model.
    \item In the well-balanced case, if $\bm x \sim \mathrm{Normal}(\bm 0, \bm I_d)$, which implies $\zeta =1/2$, $\gamma =e $, and $\eta =1$, then the sample complexity requirement of $n$ becomes $\mathcal{O}((\sigma_z^2\vee 1)dk^8)$. 
    \item Consider the case when $\{\bm \beta_j\}_{j=1}^k$ and $\{\bm \alpha_l\}_{l=1}^k$ are sampled independently from $\mathcal{N}(0, I_{d+1})$. In this case, we expect the model parameters to be highly uncorrelated. This makes the linear regions of the first max-affine function partition all other linear regions of the second max-affine function, thus leading to $r \approx 1$. In this case, for a normalized dataset, i.e. $\eta \approx 1$, and we can have $c_3 = 1$. In other words, for this scenario, Theorem \ref{Theo:main} holds for any $k \geq 2$ (the case when $k = 1$ reduces to linear regression). In general, the value of $c_3$ depends on the quantity $\phi(k;r)$ shown in \eqref{eq:bjt_bound_final}. 
\end{enumerate}
Another important observation is that the class of DoMA functions captured by Theorem \ref{Theo:main} is characterized by $\Theta(\kappa,\pi_{\min},r)$ presented in Definition \ref{eq:Theta_class}. The conditions in \eqref{eq:defkappa} and \eqref{eq:def_pimin_pimax} are general in the sense that Theorem \ref{Theo:main} holds for any valid $\kappa>0$, and $\pi_{\min} \in (0,1]$. However, the condition on the minimum intersection probability is only general when $r=0$. The currently available proof techniques only allow the statement of Theorem \ref{Theo:main} when $r\in(1/2,1]$. This range of $r$ captures DoMA functions where each element in $\{\mathcal{C}_j(\bm \beta)\}_{j=1}^{k_1}$ is properly partitioned by $\{\mathcal{C}_l(\bm \alpha)\}_{l=1}^{k_2}$ and vice versa.

Note that Theorem \ref{Theo:main} provides the performance guarantees of the parameter estimation by ABGD. 
Ultimately, it is of interest to estimate the ground-truth DoMA function from the available samples. This motivates the following corollary, whose proof is deferred to Appendix \ref{append:proof:theo_general}.  
\begin{corollary}\label{theo:general}
    Instate the assumptions of Theorem \ref{Theo:main}, let $(\hat{\bm \beta},\hat{\bm \alpha})$ denote the final parameter estimate by ABGD, and $k =k_1\vee k_2$. Then it holds with probability at least $1-ke^{-d}$ that
\begin{equation}
\label{eq:cor_res}
\begin{aligned}
 \mathbb E_{\bm x} \left[\left|f(\bm x; \hat{\bm \beta}, \hat{\bm \alpha} ) - f(\bm x; \bm \beta^\star, \bm \alpha^\star ) \right|^2 \right]
 \leq
  \frac{C\sigma_z^2 d k \log(n/d)}{n},
\end{aligned}
\end{equation}
where $f$ is the DoMA function defined in \eqref{eq:classconf} and $\bm x \in \mathbb R^d$ is a random vector satisfying Assumptions \ref{assum:subg} and \ref{assum:anti}, independent of everything else. 
\end{corollary}
This corollary translates the parameter estimation error in Theorem \ref{Theo:main} to the generalization error in predicting the ground-truth DoMA function in the expected squared loss and inherits the optimal linear scaling of $n$ with respect to $d$.
%%%%%%%%%%%%%%%%%%%%%%%%%%%%%%%%%%%%%%%%%%%%%%%%%%%%%%%%%%%%%%%%%%%%%%%%%%%%%%%%%
\begin{figure*}[t]
    \centering
    \includegraphics[width=0.95\linewidth]{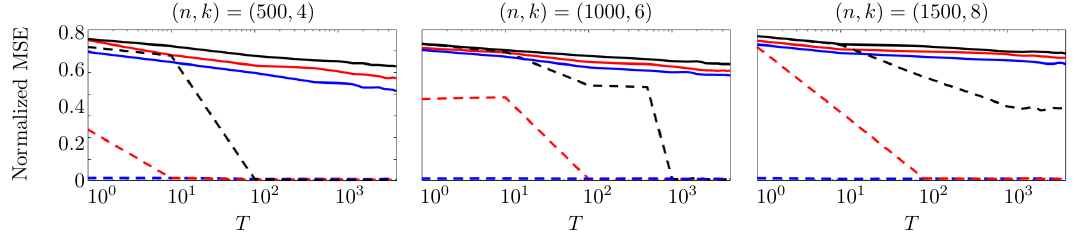}
    \caption{Median of the normalized mean squared test error across $50$ Monte Carlo iterations for initialization (Algorithm \ref{algo:Initialization}, solid) followed by ABGD (Algorithm \ref{algo:AlternatingGD}, dashed), when $d\in \{50,60,70\}$ (blue, red, black) for different $(n,k)$ pairs where $k_1=k_2=k$.  }
    \label{fig:MCV}
\end{figure*}

\section{Initializing ABGD} \label{Sec:Spectral}
Since the square loss in \eqref{eq:loss_def_max0} poses a non-convex estimation problem, the local convergence result by Theorem \ref{Theo:main} requires a suitable initialization within the basin of attraction of the ground truth DoMA function parameters. 
This section surveys recent techniques that might be adapted to obtain the desired initial estimates for DoMA regression. 
Toward this end, we first note that the DoMA function in \eqref{eq:classconf} is a special instance of the class of multi-index models, i.e. there exists a nonlinear multivariate function $g: \mathbb{R}^{k_1+k_2} \to \mathbb{R}$ such that $f$ in \eqref{eq:classconf} is equivalently rewritten as 
$
f(\bm x) = g(\bm\Theta^\mathsf{T} \bm x ) 
$
with 
$
\bm \Theta = [
        \bm \beta_1,\ldots, \bm \beta_{k_1}, \bm \alpha_1, \ldots,\bm \alpha_{k_2}
]
$.
Various computational methods to estimate the column space of $\bm \Theta$ and their theoretical analysis have been developed \citep{li1991sliced,xia2002adaptive,ghosh2021max}. 
Most of these methods were originally developed for semi-parametric learning. However, a quantized search over the individual parameter vectors in the reduced subspace is feasible when $k_1+k_2$ is small and $g$ is known. Below, we discuss a subset of the subspace methods relevant to DoMA regression.  

\citet{li1991sliced} proposed sliced inverse regression (SIR). 
When the covariates follow an elliptically symmetric distribution (e.g. Gaussian, logistic, and Student's $t$), SIR estimates $\mathrm{span}(\bm \Theta)$ via principal component on the empirical version of $\mathrm{Cov}(\mathbb{E}\left[\bm x|y\right])$ where $y \in \mathbb{R}$ and $\bm x \in \mathbb{R}^d$ denote the target and covariates.   
On the other hand, \citet{xia2002adaptive} developed the outer product of gradients (OPG) based on the observation that the sample average of the outer product of the gradient, $\nabla_{\bm x} f(\bm x) = \bm \Theta \nabla_{\bm u} \left. g(\bm u) \right|_{\bm u = \bm \Theta^\mathsf{T}\bm x}$, spans the column space of $\bm \Theta$. However, the guarantees for OPG \citep[Condition 5 (b)]{xia2002adaptive} and SIR \citep[Assumption 1]{tan2018convex} require a smoothness condition on $\mathbb{E}\left[\bm x|y\right]$ which is not trivial to show for the DoMA functions. We motivate a proof direction for such conditions in the A

\begin{algorithm}[tb]
   \caption{(Initialization Algorithm)}
   \label{algo:Initialization}
\begin{algorithmic}
   \STATE {\bfseries Input:} Dataset $\{\bm x_i, y_i\}_{i=1}^{n}$, model orders $(k_1,k_2)$, and repetition count $T$.
   \vspace{0.1cm}
   \STATE \textbf{1.} Construct $\hat{\bm m}_1$ and $\hat{\bm M}_2$, the data-driven versions of $\bm m_1$ and $\bm M_2$ and let $\hat{\bm M} = \hat{\bm m}_1 \hat{\bm m}_1^{\top} + \hat{\bm M}_2$.
   \vspace{0.2cm}
   \STATE \textbf{2.} Construct $\hat{\bm V} \in \mathbb R^{d\times (k_1+k_2-1)}$ as the dominant eigen-space of $\hat{\bm M}$.
   \vspace{0.2cm}
   \STATE \textbf{3.} Randomly sample the subspace spanned by $\hat{\bm V}$ to get a candidate pair $(\tilde{\bm \beta}, \tilde{\bm \alpha})$.
   \vspace{0.2cm}
   \STATE \textbf{4.} Refine $(\tilde{\bm \beta}, \tilde{\bm \alpha})$ using a few ABGD iterations.
   \vspace{0.2cm}
   \STATE \textbf{5.} repeat "3." and "4." $T$ times and select $(\bm \beta^0, \bm \alpha^0)$ that minimizes \eqref{eq:loss_def_max0}.
   \vspace{0.1cm}
   \STATE {\bfseries Output:} Initial estimates $\bm \beta^0 = [\bm \beta_1^0; \ldots; \bm \beta_{k_1}^0]$, and $\bm \alpha^0 = [\bm \alpha_1^0; \ldots; \bm \alpha_{k_2}^0]$.
\end{algorithmic}
\end{algorithm}

On the other hand, the spectral method, which does not require the smoothness condition, has been widely used to obtain initial estimates for inverse problems, e.g. phase retrieval \citep{candes2015phase} and its generalization to max-affine regression \citep{ghosh2021max}.
In a nutshell, the spectral method constructs an empirical estimate $\hat{\bm M}$ of a proxy matrix $\bm M$ with the observations of the covariate-target pairs so that $\bm M$ spans the column space of $\bm \Theta$. 
A basis for $\mathrm{span}(\bm \Theta)$ is estimated as the dominant eigenvectors via principal component analysis on $\hat{\bm M}$. 
In the remainder, we show that the proxy matrix $\bm M$ proposed by \citet{ghosh2021max} for max-affine regression can also be used for its generalization to the DoMA function in \eqref{eq:classconf}. 
The following lemma, proved in Appendix~\ref{append:proof:m1m2}, illustrates that the proxy matrix $\bm M$ spans a subspace spanned by the pairwise difference $[\bm \beta_j]_{1:d} - [\bm \alpha_l]_{1:d}$ for $(j,l) \in [k_1] \times [k_2]$ when the covariates are i.i.d. standard Gaussian. 
The Gaussianity of covariates is crucial to invoke Stein's lemma for the calculation of expectations. 

\begin{lemma} \label{lemma:M1M2}
    Let $\bm x\sim \mathrm{Normal}(\bm 0, \bm I_d)$ and $y= f(x;\bm \beta,\alpha)+z$ with $f(\cdot)$ defined from \eqref{eq:classconf} and  $\bm z\sim \mathrm{Normal}(\bm 0, \sigma_z^2\bm I_d)$ independent of $\bm x$. Define $\bm M  = \bm m_1\bm m_1^\mathsf{T} +{\bm M}_2$ with
 \begin{equation} 
     \bm m_1 = \mathbb E\left[y\bm x\right], \quad \bm M_2 = \mathbb E\left[ y\left(\bm x\bm x^\mathsf{T}- \bm I_d\right)\right].
 \end{equation}
 Then $\bm m_1$ and $\bm M_2$ are explicitly written as
 \begin{equation*}
     \bm m_1  = \sum_{(j,l)\in [k_1]\times[k_2]} \left[\bm \beta_j -\bm \alpha_l\right]_{1:d}\mathbb P\left( \bm x \in \mathcal{C}_j(\bm \beta)\cap \mathcal{C}_l(\bm \alpha)\right)
 \end{equation*}
 and 
 \begin{equation*}
     \bm M_2 = \sum_{(j,l)\in [k_1]\times[k_2]} [\bm \beta_j -\bm \alpha_l]_{1:d} \mathbb{E}\left[\mathbf{1}_{\{\bm x\in \mathcal{C}_j(\bm \beta)\cap\mathcal{C}_l(\bm \alpha)\}}\bm x \right]^\mathsf{T}.
 \end{equation*}
\end{lemma} 
Note that the DoMA function in \eqref{eq:classconf} remains the same when all parameter vectors are shifted by an arbitrary constant vector $\bm v \in \mathbb{R}^{d+1}$. 
The next lemma shows that the span of the pairwise difference of parameter vectors indeed coincides with the span of all parameter vectors shifted by one of them. 

\begin{lemma}
     \label{lemm:Span}
     Let $\mathcal{B} = \{\bm \beta_j\}_{j=1}^{k_1}$ and $\mathcal{A} = \{\bm \alpha_l\}_{l=1}^{k_2}$ be subsets of $\mathbb R^{d+1}$. 
     Then we have
     \begin{equation}             
     \label{eq:span_statement}
     \mathrm{span}\left(\mathcal{B}\ominus\mathcal{A}\right) = \mathrm{span}\left(-\bm{v}+ \left(\mathcal{B}\cup \mathcal{A}\right)\right), \quad \forall \bm v \in \mathcal{B} \cup \mathcal{A}
     \end{equation}
     and
     \begin{align}
        \mathrm{dim} \left[\label{eq:rank_statement}
         \mathrm{span}\left(\mathcal{B}\ominus\mathcal{A}\right)\right] \leq k_1+k_2-1, 
     \end{align}
     where $\ominus$ denotes the Minkowski set difference, i.e. the pairwise difference of elements from each set. 
\end{lemma}
The proof of this lemma is deferred to Appendix \ref{append:proof:span} and might be of independent interest.
\begin{figure}
    \centering
    \includegraphics[width=0.5\linewidth]{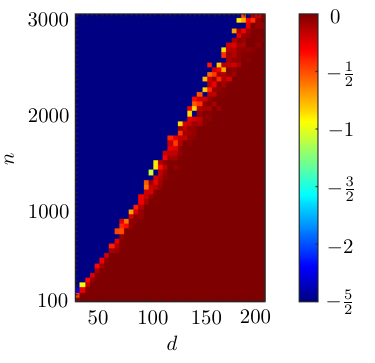}
    \caption{Median of $\log_{10}E(\hat{\bm{\beta}},\hat{\bm{\alpha}})$ for different ($n$,$d$) pairs using 50 Monte Carlo iterations for $k_1= k_2=3$ when the covariates follow the Gaussian covariate distribution with $\sigma_z=0.1$.}
    \label{fig:synthetic1}
\end{figure}
Furthermore, the proof of \citep[Theorem~2]{ghosh2021max} relies on the fact that every max-affine function is convex, which no longer holds for difference max-affine functions. Still, Lemmas \ref{lemma:M1M2} and \ref{lemm:Span} show that $\bm m_1$ and $\bm M_2$ are in the span of all the individual DoMA model parameters. This motivates us to adapt the initialization algorithm proposed by \citet{ghosh2021max} to our DoMA modeling as presented in Algorithm \ref{algo:Initialization}. This algorithm first estimates the span of all DoMA model parameters, then randomly samples this lower-dimensional subspace $T$ times to select the best candidate to initialize ABGD. The total runtime cost of this algorithm is $\mathcal{O}(nd^2+d^3 +nde^k)$. In the original setting, this algorithm is guaranteed to recover a suitable initial estimate given $\mathcal{O}(d~\mathrm{poly} k)$ samples. Given our results in Lemmas \ref{lemma:M1M2} and \ref{lemm:Span}, we conjecture the same non-asymptotic sample complexity to hold for DoMA initialization. In fact, one could also use OPG (the method by \cite{xia2002adaptive}) in place of the first step of Algorithm \ref{algo:Initialization} to construct an alternative to $\hat{\bm M}$. OPG relies on local linear approximations and is asymptotically guaranteed to succeed with $\mathcal{O}(d)$ samples.  
This means our framework, compared to existing methods, is the first to scale efficiently to large datasets (since the compute cost is linear in the sample size $n$). Also, the cubic dependence on $d$ can be reduced to a quadratic by implementing the PCA step of Algorithm \ref{algo:Initialization} in a randomized fashion \cite{rokhlin2010randomized}. Furthermore, our runtime scales exponentially in the fixed parameter $k \ll \min(n,d)$. Along the same lines, training ReLU neural networks, which is also piecewise linear, exhibits a similar exponential dependence on the network size and is NP-hard \cite{froese2023training}.
However, this is quite reasonable when we observe that $k$ is the number of linear segments in each max-affine function in DoMA and is not the number of linear segments in the overall DoMA model. Since each linear region of the first max-affine function can partition every linear region of the second max-affine function, the overall DoMA function has $k_{\mathrm{DoMA}} \propto k\times k = k^2$ linear segments. This means a choice of model order leads to a quadratically larger number of segments in the overall DoMA model.

\section{Numerical Results} \label{Sec:Numerical}
This section discusses the empirical performance of ABGD.
First, we show the empirical success of Algorithm \ref{algo:Initialization} in initializing ABGD. Figure \ref{fig:MCV} shows the median of the normalized mean squared error (MSE) across $50$ Monte Carlo iterations for different values of $(n,d,k)$ with initialization only (solid) and after ABGD (dashed). It is clear that when the number of random samples by the initialization algorithm, $T$, exceeds a certain threshold, the test MSE becomes negligible, indicating successful initialization and exact recovery of the model parameters by ABGD. 
\begin{figure}
    \centering
    \includegraphics[width=0.5\linewidth]{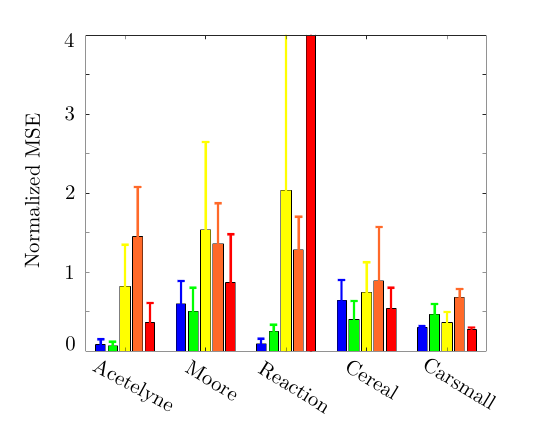}
    \caption{ Performance comparison of ABGD (blue), nonparametric difference of convex (green), $k$-nearest neighbors regression (yellow), multivariate adaptive regression splines (orange), and ReLU DNN (red) showing normalized MSE and $95\%$ confidence intervals.}
    \label{fig:synthetic2}
\end{figure}
Furthermore, we corroborate the theoretical guarantees in Theorem \ref{Theo:main} with simulations.  
Notice that adding any constant vector $\bm v \in \mathbb{R}^{d+1}$ to all model parameters, i.e. $\{\bm \beta_j+\bm v\}_{j=1}^{k_1}$ and $\{\bm \alpha_l+\bm v\}_{l=1}^{k_2}$, would result in an equivalent DoMA model in \eqref{eq:def_model_xi}. 
Furthermore, any permutation of the parameters of each max-affine model would also lead to an equivalent PL model. Therefore, to measure the performance of the parameter estimation by Algorithm \ref{algo:AlternatingGD}, we first solve for the optimal shift and permutation ambiguities as
\begin{align}\label{eq:ambg_shift}
   \mathrm{min} \Bigg\{  &\sum_{j=1}^{k_1}\lVert \hat{\bm \beta}_{\pi_1(j)} -\bm \beta^\star_j +\bm v\rVert_2^2 +  \sum_{l=1}^{k_2}\lVert \hat{\bm \alpha}_l -\bm \alpha_{\pi_2(l)}^\star+\bm v\rVert_2^2: \bm v \in \mathbb{R}^{d+1}, \pi_1 \in \mathfrak{S}_{k_1}, \pi_2 \in \mathfrak{S}_{k_2} \Bigg\},
\end{align}
where $\mathfrak{S}_k$ denotes the set of all permutations of $[k]$.
Then the relative error, $E(\hat{\bm{\beta}}, \hat{\bm \alpha})$, is defined as the minimum of \eqref{eq:ambg_shift} normalized by $\|[\bm \beta^\star; \bm \alpha^\star]\|_2^2$. 
Figure~\ref{fig:synthetic1} illustrates the median over $50$ Monte Carlo iterations of the relative error $E(\hat{\bm{\beta}}, \hat{\bm \alpha})$ when the covariate dimension $d$ varies versus the number of samples $n$ for $k_1=k_2=3$. 
The covariates and noise samples were generated following $\mathrm{Normal}(\bm 0, \bm I_d)$ and $\mathrm{Normal}(0, 10^{-2})$ respectively. %Gaussian distribution with $\sigma_z = 0.1$. 
Figure~\ref{fig:synthetic1} shows a sharp phase transition when the sample complexity $n$ crosses a threshold scaling linearly with ambient dimension $d$. This is consistent with the sample complexity guarantee provided by Theorem \ref{Theo:main}. This Monte Carlo experiment ran for $72$ hours on a server computer equipped with an Intel(R) Xeon(R) Gold $6330$ CPU.

Furthermore, we benchmark ABGD against relevant methods on the datasets available in MATLAB's Statistics and Machine Learning Toolbox. 
The benchmark datasets are of small sizes ($n \leq 10^3$ and $d<10^2$) so that the comparison can include nonparametric methods that do not scale well to large dimensions. 
%Since several methods we benchmark against are nonparametric, thus suffering from untractable requirements on the sample size and computation in high dimensions, we choose datasets that best fit their regime of operation ($n \leq 10^3$ and $d<10^2$).
The methods we compare are: (i) ABGD on the DoMA model with $k_1 =k_2 \leq 4$, (ii) non-parametric difference of convex (NDC) regression \citep{siahkamari2020piecewise}, (iii) $k$-nearest neighbors with $k\leq 10$, (iv) Multivariate Adaptive Regression Splines \citep{jekabsons2011areslab}, and (v) $L$-layer ReLU DNN with $L\leq 4$ and $4$ neurons per layer. The choice of layers for the ReLU DNN was guided by the previously discussed equivalence with DoMA functions ($L \leq \lceil \log_2\max(k_1,k_2)\rceil +2 = 4$). A random train-test split of ($80\%,20\%$) was used, and the hyperparameters of each method were tuned via $5$-fold cross-validation. We repeat this experiment $5$ times and report the mean test error and the 
$95\%$ confidence intervals in Figure \ref{fig:synthetic2}. It is observed that ABGD performs competitively with NDC and outperforms existing methods. 
%It is also worth mentioning that in the high-dimensional setting, such as Figure \ref{fig:synthetic} (left), NDC is computationally intractable and requires a sample size scaling exponentially with the ambient dimension \citep[Theorem 3]{siahkamari2020piecewise}. 
However, ABGD is much cheaper to solve than NDC, even with faster, more recent implementations \citep{siahkamari2022faster}. 
The implementation of ABGD and all competing methods is available in our GitHub repository \footnote{\href{https://github.com/haitham-kanj/DoMA-Regression}{https://github.com/haitham-kanj/DoMA-Regression}}.
%For such nonparametric methods, only the computational requirement can be improved using more involved training algorithms \citep{siahkamari2022faster}.
%%%%%%%%%%%%%%%%%%%%%%%%%%%%%%%%%%%%%%%%%%%%%%%%%%%%%%%%%%%%%%%%%%%%%%%%%%%%%%%%%

\section{Tropical Consequences} \label{Sec:Tropical}
This section discusses geometric consequences for PL regression from tropical algebra.
It has been shown that a DoMA function can be represented as an $L$-layer ReLU DNN where $L \leq \lceil\log_2\max(k_1,k_2)\rceil +2$ \citep[Theorem 5.4]{zhang2012sparse}. This result is based on the observation that a DoMA function is a natural object from max-plus algebra called the \textit{tropical rational function}.
Furthermore, recall that a neuron in a ReLU DNN is called \textit{inactive} if its weights do not have any influence on the output. 
We observe a similar interesting property for max-affine functions. 
Let $\mathcal{N}_{\bm \beta} = \mathrm{conv}(\{\bm \beta_j\}_{j=1}^k)$ be the \textit{Newton polytope} from tropical algebra whose vertices are model parameters. Then a property from max-plus algebra states that two max-affine functions are equivalent if they share the same Newton polytope \citep[Eq.~15]{maragos2021tropical}. Therefore, the max-affine structure can artificially deactivate any affine function by simply placing it within the convex hull of other affine functions. 
For example, in a simple max-affine function given by $\max(-x,x,x/2) = \max(-x,x)$, the linear function $x/2$ is inactive. This property implies that DoMA models are equivalent if their max-affine parts share the same Newton polytopes. With $\ominus$ denoting the Minkowski set difference, Figure \ref{fig:kchioce} provides geometric examples of this property where all model parameters interior to the normal cones (green) are inactive. 
\begin{figure}
    \centering
    \includegraphics[width =0.5\linewidth]{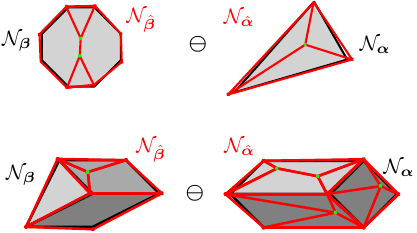}
    \caption{Consequence of overspecified orders in the DoMA regression: Newton polytope estimates (red) and groundtruth (black) with inactive parameters (green) for the 2D (top) and 3D (bottom) cases.}
    \label{fig:kchioce}
\end{figure}
This implies a major consequence for ABGD post-training, that is: any model parameter that can be written as a convex combination of other model parameters can be dropped; effectively, \textit{lossless compression} of the estimated DoMA model. This geometric result does not translate to ReLU DNNs since Newton polytopes are only defined through the difference of max-affine parametrization.

\section{Conclusion}
This paper presents ABGD, the first piecewise linear regression algorithm with the near minimax optimal sample complexity of $\tilde{\mathcal{O}}(d)$. Our method also inherits the cost efficiency of gradient-based methods. An initialization algorithm is presented and is shown to successfully initialize ABGD. We test the performance of our approach on synthetic data, corroborating the theoretical rate, and showing competitive performance with existing SoTA methods for nonlinear regression on real-world data. Furthermore, even empirically competitive methods suffer from pessimistic theoretical guarantees imposed by the nature of their analysis. Finally, a geometric interpretation of DoMA models post ABGD training is provided through tropical algebra.

\section*{Funding}
This work was supported in part by an NSF grant (Award No. CCF 19-43201).

\bibliography{ref}
\bibliographystyle{plainnat}
%USE THE BELOW OPTIONS IN CASE YOU NEED AUTHOR YEAR FORMAT.
%\bibliographystyle{abbrvnat}
%\bibliography{reference}

\appendix

\section{Technical Appendices and Supplementary Material} \label{Appendix}
\subsection{Proof of Theorem \ref{Theo:main} } \label{append:proof:theo}
Without loss of generality, we assume that $k=k_1=k_2$ to simplify the proof. 
The assertion in Theorem \ref{Theo:main} follows if 
\begin{equation}\label{eq:rec_beta}
\|\bm \beta^{t+1} - \bm \beta^\star\| \leq \tau_1 \|\bm \beta^{t} - \bm \beta^\star\| + \tau_2\|\bm \alpha^{t} - \bm \alpha^\star\|    
\end{equation}
and 
\begin{equation} \label{eq:rec_alpha}
\|\bm \alpha^{t+1} - \bm \alpha^\star\| \leq \tau
_1\|\bm \alpha^{t} - \bm \alpha^\star\| + \tau_2\|\bm \beta^{t} - \bm \beta^\star\|
\end{equation}
hold for some $\tau_1,\tau_2 \in (0,1/2)$. For notational simplicity,  we make use of the following shorthands
\begin{equation}
\label{eq:Notation}
\begin{aligned}
    &\mathcal{B}_j^t = \mathcal{C}_j(\bm \beta^t), ~ \mathcal{B}_j^\star = \mathcal{C}_j(\bm \beta^\star), \mathcal{A}_l^t = \mathcal{C}_l(\bm \alpha^t), ~  \mathcal{A}_l^\star = \mathcal{C}_l(\bm \alpha^\star), \\
    &\bm h_j^t = \bm \beta_j^t - \bm \beta_j^\star, ~ 
    \bm v^t_{jj'} = \bm \beta^t_j -\bm \beta^t_{j'}, ~  \bm v^\star_{jj'} = \bm \beta^\star_j -\bm \beta^\star_{j'}, \\
    & \bm q_l^t = \bm \alpha_l^t - \bm \alpha_l^\star, ~\bm w^t_{ll'} = \bm \alpha^t_l -\bm \alpha^t_{l'}, ~  \bm w^\star_{ll'} = \bm \alpha^\star_l -\bm \alpha^\star_{l'} 
\end{aligned}
\end{equation}
for all $j,j',l,l' \in [k]$ and $t \in \{0\}\cup \mathbb N$. 
Due to the symmetry of \eqref{eq:rec_beta} and \eqref{eq:rec_alpha}, if suffices to prove \eqref{eq:rec_beta}. The left-hand side of \eqref{eq:rec_beta} is upper bounded as 
\begin{equation} \label{eq:first_induction}
    \|\bm \beta^{t+1}- \bm \beta^\star \|_2^2 
    \leq \sum_{j=1}^k\| \bm \beta_j^t-\bm \beta_j^\star -\nabla_{\bm \beta_j}\ell(\bm \beta_j^t)\|_2^2 
    = \sum_{j=1}^k\| \bm h_j^t -\mu_j(\bm \beta^t)\nabla_{\bm \beta_j}\ell(\bm \beta_j^t)\|_2^2.
\end{equation}
The gradient in \eqref{eq:first_induction} can be expanded as
\begin{align}\label{eq:gradient_expansion}
    \nabla_{\bm \beta_j}\ell(\bm \beta_j^t) =&\frac{1}{n}\sum_{i=1}^n\mathbf{1}_{\{\bm x_i\in \mathcal{B}^t_j\}}\Big(\langle\bm \xi_i,\bm \beta_j^t\rangle -\max_{q\in [k]}\langle\bm \xi_i, \bm \beta^\star_q\rangle - \max_{l\in [k]}\langle\bm \xi_i, \bm \alpha^t_l\rangle + \max_{p\in [k]}\langle\bm \xi_i, \bm \alpha^\star_p\rangle \Big) \bm \xi_i \nonumber \\
    =& \underbrace{\frac{1}{n}\sum_{i=1}^n\mathbf{1}_{\{\bm x_i\in \mathcal{B}^t_j\}}\left(\langle\bm \xi_i,\bm \beta_j^t\rangle -\max_{q\in [k]}\langle\bm \xi_i, \bm \beta^\star_q\rangle \right) \bm \xi_i}_{\bm a_j^t} \nonumber \\
    &+ \underbrace{\frac{1}{n}\sum_{i=1}^n\mathbf{1}_{\{\bm x_i\in \mathcal{B}^t_j\}}\left(- \max_{l\in [k]}\langle\bm \xi_i, \bm \alpha^t_l\rangle + \max_{p\in [k]}\langle\bm \xi_i, \bm \alpha^\star_p\rangle \right) \bm \xi_i}_{\bm b_j^t}. 
\end{align}
Plugging \eqref{eq:gradient_expansion} into \eqref{eq:first_induction}, we obtain
\begin{align} \label{eq:Induction_second}
    \|\bm \beta^{t+1}- \bm \beta^\star \|_2^2 &\leq \sum_{j=1}^k \left(\|\bm h_j^t - \mu_j(\bm \beta^t)\bm a_j^t \|_2+\|\mu_j^t \bm b_j^t \|_2\right)^2 \nonumber\\
    &\leq 2 \sum_{j=1}^k \|\bm h_j^t - \mu_j(\bm \beta^t)\bm a_j^t \|^2_2+2\sum_{j=1}^k\|\mu_j(\bm \beta^t) \bm b_j^t \|_2^2 \nonumber \\
    & \leq \tau_1 \| \bm \beta^t -\bm \beta^\star\| +2\sum_{j=1}^k\|\mu_j(\bm \beta^t) \bm b_j^t \|_2^2,
\end{align}
where the last inequality follows from \citep[Theorem 2.3]{kanj2024sparse} since it is equivalent to a single gradient step in max-affine regression. Therefore, in order for \eqref{eq:rec_beta} to follow from \eqref{eq:Induction_second}, it suffices to show that 
\begin{equation}\label{eq:Induction_alpha_req}
    \sum_{j=1}^k\|\mu_j(\bm \beta^t) \bm b_j^t \|_2^2 \leq \tau_2\|\bm q_l^t \|_2^2
\end{equation}
for some $\tau_2 \in (0,1/2)$.
Next, we can rewrite $\bm b_j^t $ as 
\begin{align}\label{eq:bjt_decomposition}
    \bm b_j^t =& \frac{1}{n}\sum_{i=1}^n\sum_{l,p}\mathbf{1}_{\{\bm x_i\in \mathcal{B}_j^t \cap \mathcal{A}_l^t \cap \mathcal{A}_p^\star\}}\langle\bm \xi_i, \bm \alpha^\star_p - \bm \alpha^t_l\rangle  \bm \xi_i \nonumber\\
    =&\frac{1}{n}\sum_{i=1}^n\sum_{l,p}\mathbf{1}_{\{\bm x_i\in \mathcal{B}_j^t \cap \mathcal{A}_l^t \cap \mathcal{A}_p^\star\}}\langle\bm \xi_i, \bm w^\star_{pl} - \bm q^t_l\rangle  \bm \xi_i\nonumber \\
     =& \frac{1}{n}\sum_{i=1}^n\sum_{\substack{l,p \\ l \neq p}}\mathbf{1}_{\{\bm x_i\in \mathcal{B}_j^t \cap \mathcal{A}_l^t \cap \mathcal{A}_p^\star\}}\langle\bm \xi_i, \bm w^\star_{pl} \rangle  \bm \xi_i 
    + \frac{1}{n}\sum_{i=1}^n\sum_{l=1}^k\mathbf{1}_{\{\bm x_i\in \mathcal{B}_j^t \cap \mathcal{A}_l^t \}}\langle\bm \xi_i,  - \bm q^t_l\rangle  \bm \xi_i \nonumber \\
    =& \underbrace{\frac{1}{n}\sum_{i=1}^n\sum_{l=1}^k\sum_{\substack{p=1\\p\neq l}}^k\mathbf{1}_{\{\bm x_i\in \mathcal{B}_j^t \cap \mathcal{A}_l^t \cap \mathcal{A}_p^\star\}}\langle\bm \xi_i, \bm w^\star_{pl} \rangle  \bm \xi_i}_{\mathcal{T}^j_1
    } 
    + \underbrace{\frac{1}{n}\sum_{i=1}^n\sum_{l=1}^k \sum_{\substack{j'=1\\j'\neq j}}^k\mathbf{1}_{\{\bm x_i\in \mathcal{B}_j^t \cap\mathcal{B}_{j'}^\star\cap \mathcal{A}_l^t \}}\langle\bm \xi_i,  - \bm q^t_l\rangle\bm \xi_i}_{\mathcal{T}^j_2} \nonumber \\
    & + \underbrace{\frac{1}{n}\sum_{i=1}^n\sum_{l=1}^k\mathbf{1}_{\{\bm x_i\in \mathcal{B}_j^t \cap \mathcal{B}_j^\star\cap \mathcal{A}_l^t \cap \mathcal{A}_l^\star \}}\langle\bm \xi_i,  - \bm q^t_l\rangle  \bm \xi_i}_{\mathcal{T}^j_3} 
    + \underbrace{\frac{1}{n}\sum_{i=1}^n\sum_{l=1}^k\sum_{\substack{p=1\\p\neq l}}^k\mathbf{1}_{\{\bm x_i\in \mathcal{B}_j^t \cap \mathcal{B}_j^\star\cap \mathcal{A}_l^t \cap {\mathcal{A}_p^\star} \}}\langle\bm \xi_i,  - \bm q^t_l\rangle  \bm \xi_i}_{\mathcal{T}^j_4}.
\end{align}

The summands $\mathcal{T}^j_1$, $\mathcal{T}^j_2$, and $\mathcal{T}^j_4$ in the upper bound in \eqref{eq:bjt_decomposition} are uniformly upper-bounded with high probability as shown in the following lemma, the proof of which is deferred to Appendix \ref{append:proof:tau}.
\begin{lemma}\label{lemm:Tau_all}
    Instate of the assumptions of Theorem \ref{Theo:main}. It holds with probability $1-ke^{-d}$ that
    \begin{equation*} 
       \max(  \|\mathcal{T}^j_1\|, \|\mathcal{T}^j_2\|,\|\mathcal{T}^j_4\|)^2 \leq \frac{\pi_{\min}^{2(1+\zeta^{-1})}}{k}\sum_{l=1}^k\|\bm q_l^t\|^2.
    \end{equation*}
\end{lemma}

The remaining summand $\mathcal{T}^j_3$ is rewritten as 
\[
\mathcal{T}^j_3 = \frac{1}{n}\sum_{l=1}^k\bm E_l^j {\bm E_l^j}^\mathsf{T} (-\bm q^t_l ) 
\] where \[
\bm E_l^j := \begin{bmatrix} \mathbf{1}_{\{\bm x_1\in \mathcal{B}_j^t \cap \mathcal{B}_j^\star\cap \mathcal{A}_l^t \cap \mathcal{A}_l^\star\}} \bm \xi_1, ~ \dots, ~ \mathbf{1}_{\{\bm x_n\in \mathcal{B}_j^t \cap \mathcal{B}_j^\star\cap \mathcal{A}_l^t \cap \mathcal{A}_l^\star \}} \bm \xi_n \end{bmatrix} \in \mathbb{R}^{(d+1)\times n}.
\]
Therefore, it follows that
\begin{align*} 
   \left \|\mathcal{T}^j_3\right\|^2_2 \leq \sum_{l=1}^k\left\|\frac{1}{\sqrt{n}}\bm E_l^j \right\|^4_2 \left\|\bm q_l^t  \right\|_2^2.
\end{align*}
For every fixed $l \in [k]$, it holds with probability at least $1-ke^{-d}$ that
\begin{align}\label{eq:res_rip}
    \left\|\frac{1}{\sqrt{n}}\bm E_l^j \right\|^4_2 
    &= \lambda^2_{\max} \left( \frac{1}{n}\bm E_l^j {\bm E_l^j}^\mathsf{T}\right) \nonumber \\&
    \leq C(\eta^4\vee 1) \Bigg[\Bigg(\underbrace{\mathbb{P}(\bm x \in \mathcal{B}_j^\star \cap \mathcal{A}_l^\star) \log\left(\frac{e}{\mathbb{P}(\bm x \in \mathcal{B}_j^\star \cap \mathcal{A}_l^\star)}\right)}_{\lambda_{jl}}\Bigg)^2
    +k^{-3}\pi_{\min}^{4(1+\zeta^{-1})} \Bigg]
\end{align}
since \eqref{eq:cond:lem:lwb_gradient} implies that the result in each of Lemmas \ref{lemm:indicator} and \ref{lemm:improved_alpha} holds with probability $1-ke^{-d}/2$. 
Since $\bm \beta^\star$ and $\bm \alpha^\star$ satisfy Definition \eqref{eq:Theta_class} where $\bm \beta, \bm \alpha$ substituted by $\bm \beta^\star, \bm \alpha^\star$, it follows that 
\begin{equation} \label{eq:worst_lambda}
\max_{l \in [k]} \lambda_{jl} \leq k^{-r} \mathbb{P}(\bm x \in \mathcal{B}_j^\star) \log\left(\frac{e k^{r}}{\mathbb{P}(\bm x \in \mathcal{B}_j^\star)}\right), \quad \forall j \in [k]. 
\end{equation}
The inequality in \eqref{eq:worst_lambda} is proved as follows: Fix $j \in [k]$. The left-hand side of \eqref{eq:worst_lambda} is written as the solution to the following optimization problem 
\begin{equation}
    \label{optimization}    \arraycolsep=5pt\def\arraystretch{1.5}
    \begin{array}{ll} \displaystyle
    \mathop{\rm maximize}_{\{p_l\}_{l=1}^k} & \displaystyle \max_{l \in [k]} p_l \log(e/p_l) \\
    \mathrm{subject~to} & 
    \displaystyle \sum_{l=1}^k p_l = \mathbb{P}(\bm x \in \mathcal{B}_j^\star) \\ 
    & 0 \leq p_l \leq k^{-r} \mathbb{P}(\bm x \in \mathcal{B}_j^\star), \quad l \in [k].
    \end{array}
\end{equation}
Since $x \mapsto x\log(e/x)$ is monotonically increasing when $x \in [0,1]$, the maximizer to \eqref{optimization} also maximizes the following optimization
\begin{equation}
    \label{eq:optimization2}    \arraycolsep=5pt\def\arraystretch{1.5}
    \begin{array}{ll} \displaystyle
    \mathop{\rm maximize}_{\{p_l\}_{l=1}^k} & \displaystyle\max_{l \in [k]} p_l \\
    \mathrm{subject~to} & 
    \displaystyle \sum_{l=1}^k p_l = \mathbb{P}(\bm x \in \mathcal{B}_j^\star) \\ 
    & 0 \leq p_l \leq k^{-r} \mathbb{P}(\bm x \in \mathcal{B}_j^\star), \quad l \in [k].
    \end{array}
\end{equation}
Since $r \in [0,1]$, the feasible set of \eqref{eq:optimization2} is non-empty and its maximizer should satisfy $p_l = k^{-r} \mathbb{P}(\bm x \in \mathcal{B}_j^\star)$ for some $l \in [k]$ thereby yielding the assertion in \eqref{eq:worst_lambda}.
Next, we show that the step size is upper-bounded with probability at least $1-ke^{-d}$ as
\begin{align}
\label{eq:mu_bound}
     \mu_j(\bm \beta^t) &:= \frac{1}{\frac{1}{n}\sum_{i=1}^{n}\mathbf{1}_{\{\bm x_i \in \mathcal{B}^t_j\}}} \leq \frac{1}{\mathbb P(\bm x\in \mathcal{B}^t_j) - k^{-3/2}\pi_{\min}^{2(1+\zeta^{-1})}} \nonumber\\
     &\leq  \frac{1}{(1-CR^{2\zeta}k^{-2(1+\zeta^{-1})})\mathbb P(\bm x\in \mathcal{B}^\star_j) - k^{-3/2}\pi_{\min}^{2(1+\zeta^{-1})}} \leq \frac{2}{\mathbb P(\bm x\in \mathcal{B}^\star_j)}
\end{align}
where the first inequality follows from \citep[Lemma A.1]{kanj2024sparse}, the second inequality follows from \citep[Lemma B.3]{kanj2024sparse}, and the last inequality follows from a suitable choice of $R>0$.
Using the bounds in \eqref{eq:res_rip} and \eqref{eq:mu_bound} yields that
\begin{align*} 
    \sum_{j=1}^k \|\mu_j(\bm \beta^t) \mathcal{T}^j_3 \|_2^2 \leq C(\eta^4\vee1) 
    \left(k^{1-2r}\log^2(e^{-1}k^{-r}\pi_{\min})+k^{-2}\pi_{\min}^{2+4\zeta^{-1}}\right)\sum_{l=1}^k \left\| \bm q_l^t \right\|_2^2. 
\end{align*}
Finally, we derive an upper bound on the left-hand side of \eqref{eq:Induction_alpha_req} given by
\begin{align}
    \label{eq:bjt_bound_final}
    \sum_{j=1}^k\|\mu_j(\bm \beta^t) \bm b_j^t \|^2 
    &\leq 4\sum_{j=1}^k \left(\sum_{p=1}^4\|\mu_j(\bm \beta^t) \mathcal{T}^j_p\|^2\right) \nonumber\\
    & \leq \underbrace{\left[48k^{-2\zeta^{-1}}+ C(\eta^4\vee1)\left(k^{1-2r}\log^2(e^{-1}k^{-r}\pi_{\min})+k^{-4(1+\zeta^{-1})}\right)\right]}_{\phi(k;r)} \, \sum_{l=1}^k\|\bm q_l^t\|^2,
\end{align}
where $\phi(k;r)$ is a nonnegative, monotonically decreasing function of $k$ when $r>1/2$. Therefore, for sufficiently large $k$, we can have $\phi(k;r)< 1/2$ satisfying the requirement in \eqref{eq:Induction_alpha_req} which concludes the proof.

\subsection{Supporting Lemmas for the Proof of Theorem \ref{Theo:main}}
The first lemma provides an upper bound on the empirical average of the indicator function when activated by membership in the intersection of tropical open cells defined in \eqref{eq:def_caljstar}.
\begin{lemma} \label{lemm:indicator}
    Instate of the assumptions of Theorem \ref{Theo:main}. For every $l, j \in [k]$, it holds with probability at least $1- ke^{-d}$ that
    \begin{equation*}
    \frac{1}{n}\sum_{i=1}^n \mathbf{1}_{\{\bm x_i\in  \mathcal{B}_j^\star \cap \mathcal{A}_l^\star\}} \leq \mathbb{P}\left(\bm x_i\in  \mathcal{B}_j^\star \cap \mathcal{A}_l^\star\right) + k^{-3/2}\pi_{\min}^{2(1+\zeta^{-1})}.
\end{equation*}
\end{lemma}
\begin{proof}
    For every fixed $l, j \in [k]$, we have that 
\begin{equation*}
    \sum_{i=1}^n \mathbf{1}_{\{\bm x_i\in \mathcal{B}_j^t \cap \mathcal{B}_j^\star\cap \mathcal{A}_l^t \cap \mathcal{A}_l^\star\}} \leq \sum_{i=1}^n \mathbf{1}_{\{\bm x_i\in  \mathcal{B}_j^\star \cap \mathcal{A}_l^\star\}}.
\end{equation*}
Let $\mathcal{D}$ be a collection of subsets in $\mathbb R^d$.  We denote the set of vectors whose entries are the indicator functions of $\mathcal{D}$ evaluated at samples $\{\bm x_i\}_{i=1}^n \subset \mathbb{R}^d$ by 
\begin{equation*}
    \mathcal{H} (\mathcal{C}, \{\bm x_i\}_{i=1}^n) := \left\{\left(\mathbf{1}_{\{\bm x_1\in C\}},\ldots,\mathbf{1}_{\{\bm x_n\in C \}}\right): C\in \mathcal{D}\right\}.
\end{equation*}
The Sauer-Shelah lemma \citep[Section 3]{mohri2018foundations} implies 
\begin{equation*}
 \sup\limits_{\bm x_1,\ldots,\bm x_n}\left|\mathcal{H} (\mathcal{D}, \{\bm x_i\}_{i=1}^n)\right| \leq \left( \frac{en}{\vcdim(\mathcal{D})}\right)^{\vcdim(\mathcal{D})},
\end{equation*}
where $\vcdim(\mathcal{D})$ denotes the Vapnik-Chervonenkis dimension of $\mathcal{D}$.
Let $\mathcal{P}_{k,d}$ collect the intersections of $k$ halfspaces in $\mathbb R^d$. 
Since the VC-dim of a single halfspace in $\mathbb R^d$ is $d+1$ \citep[Theorem A]{csikos2018optimal}, we have
\begin{equation*}
    \sup_{\bm x_1,\dots,\bm x_n}\left| \mathcal{H} (\mathcal{P}_{k,d} , \{\bm x_i\}_{i=1}^n) \right| \leq \left( \frac{en}{d+1}\right)^{k(d+1)}.
\end{equation*}
We also have by \citep[Theorem A]{csikos2018optimal} that $\vcdim(\mathcal{D}_1\cap \mathcal{D}_2) \leq C\vcdim(\mathcal{D})$ for every $\mathcal{D}_1, \mathcal{D}_2\subseteq \mathcal{D}$.  
Therefore, using \citep[Lemma SM1.3]{kim2024max}, we have with probability at least $1-\delta$ that 
\begin{align*}
    \sup_{\mathcal{B}_j^\star, \mathcal{A}_l^\star \in \mathcal{P}_{k,d}} \left| \frac{1}{n}\sum_{i=1}^n \mathbf{1}_{\{\bm x_i\in  \mathcal{B}_j^\star \cap \mathcal{A}_l^\star\}} -  \mathbb{P}\left(\bm x_i\in  \mathcal{B}_j^\star \cap \mathcal{A}_l^\star\right) \right| 
    \leq C\sqrt{\frac{\log(1/\delta)+ kd\log(d/n)}{n}}.
\end{align*}
We proceed with the choice of $\delta = ke^{-d}$.
Then, the sample complexity requirement in \eqref{eq:cond:lem:lwb_gradient} yields the assertion in \eqref{lemm:indicator}.
\end{proof}

The next lemma establishes an upper bound on the spectral norm of the empirical outer product when only a subset of summands are active.
\begin{lemma} \label{lemm:improved_alpha}
    Let $\delta \in (0,1/e)$, $\alpha \in (0,1)$, $\{\bm x_i\}_{i=1}^n$ be independent copies of $\eta$-sub-Gaussian vector $\bm x \in \mathbb{R}^d$. Then it holds with probability at least $1-\delta$ that
    \begin{align} \label{eq:improved_lemm_alpha}
    \sup_{\mathcal{I}:|\mathcal{I}|\leq \alpha n} \lambda_{\max}\bigg(\frac{1}{n}\sum_{i\in \mathcal{I}}&[\bm x_i;1][\bm x_i;1 ]^\mathsf{T}\bigg)
        \leq C (\eta^2 \vee 1) \alpha\log(e/\alpha),
    \end{align}
    provided
    \begin{equation} \label{eq:lemma_new_n}
        n\geq d \vee \frac{\log(1/\delta)}{\alpha}.
    \end{equation}
\end{lemma}
Lemma \ref{lemm:improved_alpha} improves on the previous result \citep[Theorem 5.7]{tan2019phase} by reducing the dependence on $\alpha$ from $\sqrt{\alpha}$ to $\alpha \log(e/\alpha)$ in \eqref{eq:improved_lemm_alpha}. 

\begin{proof}[Proof of Lemma \ref{lemm:improved_alpha}]    
    We follow most of the arguments in the proof by \citet[Theorem 5.7]{tan2019phase} except a sharpened upper estimate of one geometric quantity. 
    Define the set of vectors with at most $\alpha n$ support elements as
    \begin{equation*}
        \mathcal{S}_{\alpha}^n := \left\{ \bm s \in \{0,1\}^n :  \|\bm s\|_1 \leq \alpha n\right\},
    \end{equation*}
    where $\alpha n \in \mathbb N$ is assumed without loss of generality.
    Define a stochastic process ($Y_{\bm s,\bm v}$) indexed by $\bm s\in \mathcal{S}_{\alpha}^n$ and $\bm v \in B_2^d$ as
    \begin{equation*}
        Y_{\bm s,\bm v} = \sum_{i=1}^n [\bm s]_i \langle \sqrt{d} \bm a_i, \bm v\rangle^2, 
    \end{equation*}
    where $\{\bm a_i\}_{i=1}^n$ are independent copies of $\bm a$ that is uniformly distributed over the unit sphere $\mathbb{S}^{d-1}$ in $\ell_2^d$. 
    \citet[Lemma 5.4]{tan2019phase} showed that $Y_{\bm s,\bm v}$ has mixed tail increments with respect to the metrics $d_1$ and $d_2$ respectively induced by the corresponding norms 
    $\tnorm{(\bm s, \bm v)}_1 := \|\bm s\|_\infty\vee \|\bm v\|_2$ and $\tnorm{(\bm s, \bm v)}_2 :=\|\bm s\|_2\vee(\sqrt{\alpha n} \|\bm v \|_2)$, i.e. there exist absolute constants $C_1,C_2 > 0$ such that
    \begin{equation*}
        \mathbb P\left(|Y_{\bm s, \bm v}-Y_{\bm s', \bm v'}| \geq  C_1\left[\sqrt{u}d_2((\bm s,\bm v),(\bm s',\bm v'))+ ud_1((\bm s, \bm v),(\bm s', \bm v'))\right]  \right) 
        \leq C_2 e^{-u}
    \end{equation*}
    holds for every $(\bm s, \bm v),(\bm s', \bm v') \in \mathcal{S}_{\alpha}^n\times B_2^d$. 
    Therefore, due to the chaining bound by \citet[Theorem 5]{dirksen2015tail}, the supremum of $Y_{\bm s, \bm v}$ is upper-bounded as  
    \begin{align*}
\sup_{\begin{subarray}{c} \bm s, \bm s'\in \mathcal{S}_\alpha^n \\ \bm v,\bm v'\in B_2^n \end{subarray}} \left|Y_{\bm s,\bm v} - Y_{\bm s',\bm v'}\right| &\leq C \Big( \gamma_2(\mathcal{S}_{\alpha}^n\times B_2^d,d_2) + \gamma_1(\mathcal{S}_{\alpha}^n\times B_2^d,d_1) \\
&\quad + \sqrt{u}\mathrm{diam}(\mathcal{S}_{\alpha}^n\times B_2^d,d_2)  + u\mathrm{diam}(\mathcal{S}_{\alpha}^n\times B_2^d,d_1) \Big),
\end{align*}
    with probability at least $1-e^{-u}$ for any $u \geq 1$ where $\gamma_1$ and $\gamma_2$ denote the Talagrand functionals \citep[Section 2.3]{talagrand2014upper}, and $\mathrm{diam}(\mathcal{S}_{\alpha}^n\times B_2^d,d_k)$ denotes the diameter of $\mathcal{S}_{\alpha}^n\times B_2^d$ with respect to the metric $d_k$ for $k = 1,2$. 
    We use the calculations on $\gamma_1$ and the diameters by \citet{tan2019phase} given by  
    \begin{equation}
        \label{eq:gamma1_bound}
        \gamma_1(\mathcal{S}_{\alpha}^n\times B_2^d,d_1) 
        \lesssim 
        \left [\alpha n \log\left(e/\alpha\right) +d\right],\quad \mathrm{diam}(\mathcal{S}_\alpha^n\times B_2^n,d_1)=2,\quad \mathrm{diam}(\mathcal{S}_\alpha^n\times B_2^n,d_2)= 2\sqrt{\alpha n},
    \end{equation}
    where $f\lesssim g$ denotes $f\leq Cg$ for some absolute constant $C>0$. 
    The only and important distinction is to get a sharpened upper bound on $\gamma_2$ as follows. 
    Using the Dudley bound, we have 
    \begin{align*}
         \gamma_2 (\mathcal{S}_{\alpha}^n\times B_2^d,d_2)
         & \lesssim \int_0^\infty \sqrt{\log N(\mathcal{S}_{\alpha}^n\times B_2^d,d_2, u)}du \nonumber \\
        &\leq  \int_0^\infty \sqrt{\log \left[N(\mathcal{S}_{\alpha}^n,\| \cdot\|_2, u)N(B^n_2,\sqrt{\alpha n}\| \cdot\|_2, u)\right]}du \nonumber \\
         & \lesssim 
                \underbrace{\int_0^\infty \sqrt{\log N(\mathcal{S}_{\alpha}^n,\|\cdot\|_2, u)}du}_{L_1}
                   + \underbrace{\int_0^\infty \sqrt{\log N( B_2^d,\sqrt{\alpha  n}\|\cdot\|_2, u)}du}_{L_2},
    \end{align*}
    where $N$ denotes the covering number and the second inequality follows from the definition of $d_2$.
    We further proceed by obtaining an upper bound on each of $L_1$ and $L_2$. Define the set of all possible support vectors with cardinality $\alpha n$ as
    \begin{equation*}
        \mathcal{I} := \left\{ \bm \zeta \in \{0,1\}^n : \|\bm \zeta\|_1 = \alpha n\right\}.
    \end{equation*}
    Then the cardinality of this set is upper-bounded as 
    $|\mathcal{I}| = \binom{n}{\alpha n} \leq (1/\alpha)^{\alpha n}$ using Stirling's approximation.
    Let $[0,1]^{\bm \zeta}$ denote the unit cube in the coordinate set by $\bm \zeta$, i.e. 
    \[
    [0,1]^{\bm \zeta} := [\bm \zeta]_1 [0,1] \times \dots \times [\bm \zeta]_n [0,1],
    \]
    where 
    \[
    [\bm \zeta]_i [0,1] = 
    \begin{cases}
    [0,1] & \mathrm{if~} [\bm \zeta]_i = 1 \\
    \{0\} & \mathrm{if~} [\bm \zeta]_i = 0.
    \end{cases}
    \]    
    Then we have the following set inclusion relation  
    \begin{equation*}
        \mathcal{S}_\alpha^n \subset \sum_{\bm \zeta \in \mathcal{I}}[0,1]^{\bm \zeta}.
    \end{equation*}
    Therefore, it follows that
    \begin{align*}
        L_1 
        & \leq \int_0^1 \sqrt{\log N(\mathcal{S}_{\alpha}^n,\|\cdot\|_2, u) }du 
        \leq \int_0^1 \left[ \log N\left(\sum_{\bm \zeta \in \mathcal{I}}[0,1]^{\bm \zeta},\|\cdot\|_2, u\right) \right]^{1/2} du \nonumber \\
       & \lesssim \int_0^1 \sqrt{\log |\mathcal{I}|}du +\sup_{\bm \zeta \in \mathcal{I}} \int_0^1 \sqrt{\log N\left([0,1]^{\bm \zeta},\|\cdot\|_2, u\right)}du  \nonumber \\
       & \leq   \sqrt{\log |\mathcal{I}|}  + \int_0^1 \sqrt{\log N\left(\sqrt{\alpha n}B_2^{\alpha n},\|\cdot\|_2, u\right)}du  \nonumber \\
       & \leq \sqrt{\alpha n \log(1/\alpha)} + \sqrt{\alpha n} \int_0^1 \sqrt{\log N\left(B_2^{\alpha n},\|\cdot\|_2, u\right)}du \nonumber  \\
       & \leq \sqrt{\alpha n \log(1/\alpha)} + \alpha n \leq  \sqrt{n} +\alpha n ,
    \end{align*}
    where the last inequality holds for $\alpha \log(1/\alpha) \leq 1$ for all $\alpha >0$. 
    Combined with $L_2 \lesssim \sqrt{\alpha n d}$, we obtain
    \begin{equation}
        \label{eq:gamma2_bound}
        \gamma_2 (\mathcal{S}_{\alpha}^n\times B_2^d,d_2) \lesssim \alpha n + \sqrt{\alpha nd}.
    \end{equation}
Since $Y_{\bm s, \bm 0} = 0$ for all $\bm s \in \mathcal{S}_\alpha^n$, combining the previous results in \eqref{eq:gamma1_bound} and the upper bound on $\gamma_2$ in \eqref{eq:gamma2_bound} yields that, for any $u \geq 1$, it holds with probability at least $1-e^{-u}$ that
    \begin{align*} 
        \sup_{s\in \mathcal{S}_\alpha^n, \bm v\in B_2^n} Y_{s,\bm v}  &\lesssim 
          \alpha n \log(e/\alpha)+ d + \sqrt{\alpha n d} + \sqrt{un\alpha} + u 
        \lesssim \alpha n\log(e/\alpha),
    \end{align*}
where the last inequality follows from \eqref{eq:lemma_new_n}. 
\end{proof}

\subsection{Proof of Corollary \ref{theo:general} } \label{append:proof:theo_general}
Let $\hat{\bm \beta} = \bm \beta^t$ and $\hat{\bm \alpha} = \bm \alpha^t$ be the final estimates from ABGD with sufficiently large step count $t$. 
By the definition of the DoMA function in \eqref{eq:classconf}, we have
\[
f(\bm x; \bm \beta, \bm \alpha ) 
= f(\bm x; \bm \beta, \bm 0 ) - f(\bm x; \bm 0, \bm \alpha ), \quad \forall \bm \alpha, \bm \beta.
\]
Therefore, the left-hand side of \eqref{eq:cor_res} is upper-bounded as
\begin{align}\label{eq:general_split}
\mathbb E_{\bm x} \left[\left|f(\bm x; \hat{\bm \beta}, \hat{\bm \alpha} ) - f(\bm x; \bm \beta^\star, \bm \alpha^\star ) \right|^2 \right] \leq &2\mathbb E_{\bm x} \left[\left|f(\bm x; \hat{\bm \beta}, \bm 0 ) - f(\bm x; \bm \beta^\star, \bm 0 ) \right|^2 \right] \nonumber \\ &+ 2\mathbb E_{\bm x} \left[\left|f(\bm x; \bm 0,\hat{\bm \alpha} ) - f(\bm x; \bm 0,\bm \alpha^\star ) \right|^2 \right].  
\end{align}
Let $\hat{\mathcal{B}}_j =\mathcal{C}_j(\hat{\bm \beta})$ and $\hat{\mathcal{B}}_j^\star =\mathcal{C}_j(\bm \beta^\star)$ for all $j \in [k]$. 
Then the first summand of the right-hand side of \eqref{eq:general_split} is bounded from above as
\begin{align}\label{eq:beta_split}
   \mathbb E_{\bm x} &\left[\left|f(\bm x; \hat{\bm \beta}, \bm 0 ) - f(\bm x; \bm \beta^\star, \bm 0 ) \right|^2 \right]\nonumber \\ \leq &\mathbb E \left[ \left| \sum_{j,j' \in [k]} \mathbf{1}_{\{\bm x \in \hat{\mathcal{B}}_j\cap \mathcal{B}_{j'}^\star\}} \langle [\bm x;1], \hat{\bm \beta}_j - \bm \beta_{j'}^\star \rangle \right|^2 \right] \nonumber \\
   \leq& \underbrace{\mathbb E \left[ \left| \sum_{j,j' \in [k]} \mathbf{1}_{\{\bm x \in \hat{\mathcal{B}}_j\cap \mathcal{B}_{j'}^\star\}} \langle [\bm x;1], \hat{\bm h}_j \rangle \right|^2 \right]}_{\mathrm{(I)}} + \underbrace{ \mathbb E \left[ \left| \sum_{j\neq j' \in [k]} \mathbf{1}_{\{\bm x \in \hat{\mathcal{B}}_j\cap \mathcal{B}_{j'}^\star\}} \langle [\bm x;1], \bm v_{jj'}^\star \rangle \right|^2 \right] }_{\mathrm{(II)}},
\end{align}
where $\hat{\bm h}_j = \hat{\bm \beta}_j - \bm \beta_j^\star$ and $\bm v_{jj'}^\star = \bm \beta_j^\star - \bm \beta_{j'}^\star$. 
First, we derive a bound on $\mathrm{(I)}$ as 
\begin{align}
    \mathrm{(I)} =& \mathbb E \left[ \underbrace{\left( \sum_{j' \in [k]} \mathbf{1}_{\{\bm x \in \mathcal{B}_{j'}^\star\}} \right)}_{=1} \sum_{j \in [k]} \mathbf{1}_{\{\bm x \in \hat{\mathcal{B}}_j\}}  \langle [\bm x;1], \hat{\bm h}_j \rangle^2 \right] 
    = \mathbb E \left[ \sum_{j \in [k]} \mathbf{1}_{\{\bm x \in \hat{\mathcal{B}}_j\}}  \hat{\bm h}_j^\mathsf{T} [\bm x;1][\bm x;1]^\mathsf{T} \hat{\bm h}_j \right] \nonumber \\
    \leq& \sum_{j \in [k]}   \hat{\bm h}_j^\mathsf{T} \mathbb E \left[ [\bm x;1][\bm x;1]^\mathsf{T} \right] \hat{\bm h}_j 
    = \sum_{j \in [k]}   \hat{\bm h}_j^\mathsf{T} \bm I_{d+1} \hat{\bm h}_j = \sum_{j\in [k]} \left\|\hat{\bm h}_j\right\|^2 = \left\|\hat{\bm \beta} - \bm \beta^\star\right\|^2,
\end{align}
where the first equality follows from $\mathbf{1}_{\{\bm x \in \hat{\mathcal{B}}_{j_1}\cap \mathcal{B}_{j_1'}^\star\}} \mathbf{1}_{\{\bm x \in \hat{\mathcal{B}}_{j_2}\cap \mathcal{B}_{j_2'}^\star\}} \neq 0$ only if $j_1 = j_2$ and $j_1'=j_2'$, and the third equality follows from Assumption \ref{assum:subg}. 
Let $\hat{\bm v}_{jj'} = \hat{\bm \beta}_j - \hat{\bm \beta}_{j'}$. 
Then we obtain an upper bound on $\mathrm{(II)}$ as
\begin{align}
    \mathrm{(II)}=&  \mathbb E \left[ \left| \sum_{j\neq j' \in [k]} \mathbf{1}_{\{\bm x \in \hat{\mathcal{B}}_j\cap \mathcal{B}_{j'}^\star\}} \langle [\bm x;1], \bm v_{jj'}^\star \rangle \right|^2 \right] 
    =  \mathbb E  \left[ \sum_{j\neq j' \in [k]} \mathbf{1}_{\{\bm x \in \hat{\mathcal{B}}_j\cap \mathcal{B}_{j'}^\star\}} \langle [\bm x;1], \bm v_{jj'}^\star \rangle^2 \right] \nonumber \\
    \leq & \frac{2}{5 \gamma k}\left(\frac{\pi_{\mathrm{min}}}{16}\right)^{1+\zeta^{-1}} \sum_{j' \neq j \in [k]}\|\hat{\bm v}_{jj'}- \bm v_{jj'}^\star\|^2 = \frac{2}{5 \gamma k}\left(\frac{\pi_{\mathrm{min}}}{16}\right)^{1+\zeta^{-1}} \sum_{j' \neq j \in [k]}\|\hat{\bm h}_{j}- \hat{\bm h}_{j'}\|^2 \nonumber\\
     \leq & \frac{2}{5 \gamma k}\left(\frac{\pi_{\mathrm{min}}}{16}\right)^{1+\zeta^{-1}} \sum_{j' \neq j \in [k]} \left( \|\hat{\bm h}_{j}\|^2+\| \hat{\bm h}_{j'}\|^2 \right) \leq \frac{4(k-1)}{5\gamma k}\left(\frac{\pi_{\min}}{16}\right)^{1+\zeta^{-1}} \left \| \hat{\bm \beta} - \bm \beta^\star\right\|^2,
\end{align}
where the first inequality follows by taking the expectation of \citet[Lemma B.5]{kanj2024sparse}.  
Plugging $\mathrm{(I)}$ and $\mathrm{(II)}$ into \eqref{eq:beta_split} yields 
\begin{equation}\label{eq:beta_split_bound}
    \mathbb E_{\bm x} \left[\left|f(\bm x; \hat{\bm \beta}, \bm 0 ) - f(\bm x; \bm \beta^\star, \bm 0 ) \right|^2 \right] \leq \left(\frac{8(k-1)}{4\gamma k}\left(\frac{\pi_{\min}}{16}\right)^{1+\zeta^{-1}} +1\right) \cdot \left \| \hat{\bm \beta} - \bm \beta^\star\right\|^2.
\end{equation}
By symmetry, 
%of $(\hat{\bm \alpha},\bm \alpha^\star)$ and $(\hat{\bm \beta},\bm \beta^\star)$, 
we also have 
\begin{equation}\label{eq:alpha_split_bound}
    \mathbb E_{\bm x} \left[\left|f(\bm x; \bm 0,\hat{\bm \alpha} ) - f(\bm x; \bm 0,\bm \alpha^\star ) \right|^2 \right] \leq \left( \frac{8(k-1)}{4\gamma k}\left(\frac{\pi_{\min}}{16}\right)^{1+\zeta^{-1}} +1 \right) \cdot \left \| \hat{\bm \alpha} - \bm \alpha^\star\right\|^2.
\end{equation}
By plugging \eqref{eq:beta_split_bound} and \eqref{eq:alpha_split_bound} into \eqref{eq:general_split}, we obtain
\begin{align} \label{eq:general_final}
    \mathbb E_{\bm x} \left[\left|f(\bm x; \hat{\bm \beta}, \hat{\bm \alpha} ) - f(\bm x; \bm \beta^\star, \bm \alpha^\star ) \right|^2 \right] &\leq 2\left( \frac{8(k-1)}{4\gamma k}\left(\frac{\pi_{\min}}{16}\right)^{1+\zeta^{-1}} +1 \right) \cdot \left\|\begin{bmatrix} \hat{\bm \beta} \\ \hat{\bm \alpha} \end{bmatrix} - \begin{bmatrix} \bm \beta^\star \\ \bm \alpha^\star \end{bmatrix} \right\|^2_2 \nonumber \\
    &\leq \frac{C\sigma_z^2 d (k_1\vee k_2) \log(n/d)}{n},
\end{align}
where the last inequality follows from Theorem \ref{Theo:main}. For sufficiently large $t$, the first term of the right-hand side in \eqref{eq:theo_param_error} is dominated by the second term, resulting in the upper bound in \eqref{eq:general_final}.

\subsection{Proof of Lemma \ref{lemm:Tau_all}} \label{append:proof:tau}
We can rewrite the first term $\mathcal{T}_1^j$ in \eqref{eq:bjt_decomposition} as 
\begin{equation}
    \label{eq:T1}
\mathcal{T}_1^j = \frac{1}{n} \sum_{l=1}^k\sum_{\substack{p=1\\p\neq l}}^k \bm E^1_{p,l} \bm v^1_{p,l}
\end{equation}
with 
\begin{align*}
    \bm E^1_{p,l} :=[\mathbf{1}_{\{\bm x_1\in \mathcal{B}_j^t \cap \mathcal{A}_l^t \cap \mathcal{A}_p^\star\}}\bm \xi_1,\ldots,\mathbf{1}_{\{\bm x_n\in \mathcal{B}_j^t \cap \mathcal{A}_l^t \cap \mathcal{A}_p^\star\}}\bm \xi_n]
\end{align*}
and 
\begin{equation*}
     \bm v_{p,l}^1 := [\mathbf{1}_{\{\bm x_1\in \mathcal{B}_j^t \cap \mathcal{A}_l^t \cap \mathcal{A}_p^\star\}}\langle\bm \xi_1, \bm w^\star_{pl} \rangle,\ldots,\mathbf{1}_{\{\bm x_n\in \mathcal{B}_j^t \cap \mathcal{A}_l^t \cap \mathcal{A}_p^\star\}}\langle\bm \xi_n, \bm w^\star_{pl} \rangle]^\mathsf{T}.
\end{equation*}
Since the summand vectors in \eqref{eq:T1} are pairwise orthogonal due to the exclusive indicator functions, we have 
\begin{equation}\label{eq:T1_squared}
\| \mathcal{T}_1^j \|^2 \leq \frac{1}{n^2} \sum_{l=1}^k\sum_{\substack{p=1\\p\neq l}}^k \|\bm E^1_{p,l}\|^2 \|\bm v^1_{p,l}\|^2.
\end{equation}
We first upper bound the factor $\|\bm v^1_{p,l}\|^2$ in each summand in the right-hand side of \eqref{eq:T1_squared} as 
\begin{align}\label{eq:v1_proof}
   \frac{1}{n} \|\bm v^1_{p,l}\|^2 &= \frac{1}{n}\sum_{i=1}^n\mathbf{1}_{\{\bm x_n\in \mathcal{B}_j^t \cap \mathcal{A}_l^t \cap \mathcal{A}_p^\star\}}\langle\bm \xi_n, \bm w^\star_{pl} \rangle^2 \leq \frac{1}{n}\sum_{i=1}^n\mathbf{1}_{\{\bm x_n\in  \mathcal{A}_l^t \cap \mathcal{A}_p^\star\}}\langle\bm \xi_n, \bm w^\star_{pl} \rangle^2 \nonumber\\&
   \leq \frac{2}{5\gamma k}\left(\frac{\pi_{\min}}{16}\right)^{1+\zeta^{-1}}\|\bm w^t_{pl} -\bm w_{pl}^\star \|^2= \frac{2}{5\gamma k}\left(\frac{\pi_{\min}}{16}\right)^{1+\zeta^{-1}}\|\bm q^t_{l} -\bm q^t_{p} \|^2,
\end{align}
where the second inequality follows since \eqref{eq:cond:lem:lwb_gradient} invokes \citep[Lemma SM2.7]{kim2024max} with probability at least $1-ke^{-d}/2$ and the last equality follows from the definitions in \eqref{eq:Notation}. Furthermore, we have
\begin{equation}\label{eq:int1proof}
    \frac{1}{n}\sum_{i=1}^n\mathbf{1}_{\{\bm x_i\in \mathcal{B}_j^t \cap \mathcal{A}_l^t \cap \mathcal{A}_p^\star\}} \leq \frac{1}{n}\sum_{i=1}^n\mathbf{1}_{\{\bm x_i\in  \cap \mathcal{A}_l^t \cap \mathcal{A}_p^\star\}} \leq C\left(\frac{R^\zeta\pi_{\min}^{1+\zeta^{-1}}}{k}\right)^2 
\end{equation}
where the second inequality holds since \eqref{eq:cond:lem:lwb_gradient} invokes \citep[Lemma SM2.7]{kim2024max} with probability at least $1-ke^{-d}/2$. Using \eqref{eq:int1proof} and \citep[Theorem 5.7]{tan2019phase}, we have with probability at least $1-ke^{-d}/2$ that 
\begin{equation}\label{eq:E1_proof}
     \frac{1}{n}\|\bm E^1_{p,l}\|^2 \leq C(\eta^2 \vee 1)\frac{R^\zeta\pi_{\min}^{1+\zeta^{-1}}}{k} \leq \frac{\pi_{\min}^{1+\zeta^{-1}}}{k},
\end{equation}
where the last inequality follows from a suitable choice of $R>0$. Therefore, combining \eqref{eq:v1_proof} and \eqref{eq:E1_proof}, we have that
\begin{equation}
    \label{eq:T1_bound}
    \| \mathcal{T}_1^j \|^2= \frac{2}{5\gamma k^2}\left(\frac{1}{16}\right)^{1+\zeta^{-1}}\pi_{\min}^{2(1+\zeta^{-1})}\sum_{l=1}^k\sum_{\substack{p=1\\p\neq l}}^k\|\bm q^t_{l} -\bm q^t_{p} \|^2 \leq \frac{1}{5\gamma k}\left(\frac{1}{16}\right)^{1+\zeta^{-1}}\pi_{\min}^{2(1+\zeta^{-1})}\sum_{l=1}^k\|\bm q^t_{l}\|^2.
\end{equation}
Similarly, the second term $\mathcal{T}_2^j$ in \eqref{eq:bjt_decomposition} is rewritten as 
\[
\mathcal{T}_2^j = \frac{1}{n} \sum_{l=1}^k\sum_{\substack{j'=1\\j'\neq j}}^k \bm E^2_{j',l} (\bm E^2_{j',l})^\mathsf{T}(-\bm a_l^t ), 
\quad \text{where} \quad 
\bm E^2_{j',l} := [\mathbf{1}_{\{\bm x_1\in \mathcal{B}_j^t \cap \mathcal{B}_{j'}^\star \cap \mathcal{A}_l^t\}}\bm \xi_n,\ldots,\mathbf{1}_{\{\bm x_1\in \mathcal{B}_j^t \cap \mathcal{B}_{j'}^\star \cap \mathcal{A}_l^t\}}\bm \xi_n].
\]
Similar to \eqref{eq:int1proof}, we have \begin{equation}\label{eq:int2proof}
    \frac{1}{n}\sum_{i=1}^n \mathbf{1}_{\{\bm x_i\in \mathcal{B}_j^t \cap \mathcal{B}_{j'}^\star \cap \mathcal{A}_l^t\}} \leq \frac{1}{n}\sum_{i=1}^n \mathbf{1}_{\{\bm x_i\in \mathcal{B}_j^t \cap \mathcal{B}_{j'}^\star \}} \leq C\left(\frac{R^\zeta\pi_{\min}^{1+\zeta^{-1}}}{k}\right)^2,
\end{equation}
where the second inequality follows since \eqref{eq:cond:lem:lwb_gradient} invokes \citep[Lemma SM2.7]{kim2024max} with probability at least $1-ke^{-d}/2$.
Due to \citet[Theorem 5.7]{tan2019phase}, it holds with probability at least $1-ke^{-d}/2$ that
\begin{equation}\label{eq:E2_proof}
     \frac{1}{n}\|\bm E^2_{p,l}\|^2 \leq C(\eta^2 \vee 1)\frac{R^\zeta\pi_{\min}^{1+\zeta^{-1}}}{k} \leq \frac{\pi_{\min}^{1+\zeta^{-1}}}{k},
\end{equation}
where the last inequality follows from a suitable choice of $R>0$. 
Therefore, using \eqref{eq:int2proof} and \eqref{eq:E2_proof} we have 
\begin{equation}\label{eq:T3bound}
    \|\mathcal{T}_2^j\|^2 = \sum_{l=1}^k\sum_{\substack{j'=1\\j'\neq j}}^k \|\bm E^2_{j',l} \|^4\|\bm a_l^t \|^2/n^2 \leq \frac{\pi_{\min}^{2(1+\zeta^{-1})}}{k} \sum_{l=1}^k \|\bm q_l^t \|^2.
\end{equation}
Finally, the last term $\mathcal{T}_4^j$ in \eqref{eq:bjt_decomposition} is also rewritten as 
\[
\mathcal{T}_4^j = \frac{1}{n} \sum_{l=1}^k\sum_{\substack{p=1\\p\neq l}}^k \bm E^4_{p,l} (\bm E^4_{p,l})^\mathsf{T}(-\bm a_l^t ), \] where \[ 
     \bm E^4_{p,l} := [\mathbf{1}_{\{\bm x_1\in \mathcal{B}_j^t \cap \mathcal{B}_j^\star\cap \mathcal{A}_l^t \cap {\mathcal{A}_p^\star} \}}\bm \xi_1,\ldots,\mathbf{1}_{\{\bm x_n\in \mathcal{B}_j^t \cap \mathcal{B}_j^\star\cap \mathcal{A}_l^t \cap {\mathcal{A}_p^\star} \}}\bm \xi_n].
     \]
Similar to \eqref{eq:int1proof}, we have 
\begin{equation}\label{eq:int3proof}
    \frac{1}{n}\sum_{i=1}^n\mathbf{1}_{\{\bm x_i\in \mathcal{B}_j^t \cap \mathcal{B}_j^\star\cap \mathcal{A}_l^t \cap {\mathcal{A}_p^\star} \}} \leq \frac{1}{n}\sum_{i=1}^n \mathbf{1}_{\{\bm x_i\in \mathcal{A}_l^t \cap {\mathcal{A}_p^\star} \}} \leq C\left(\frac{R^\zeta\pi_{\min}^{1+\zeta^{-1}}}{k}\right)^2,
\end{equation}
where the second inequality follows since \eqref{eq:cond:lem:lwb_gradient} invokes \citep[Lemma SM2.7]{kim2024max} with probability at least $1-ke^{-d}/2$.
Due to \citet[Theorem 5.7]{tan2019phase}, it holds with probability at least $1-ke^{-d}/2$ that  
\begin{equation}\label{eq:E4_proof}
     \frac{1}{n}\|\bm E^4_{p,l}\|^2 \leq C(\eta^2 \vee 1)\frac{R^\zeta\pi_{\min}^{1+\zeta^{-1}}}{k} \leq \frac{\pi_{\min}^{1+\zeta^{-1}}}{k},
\end{equation}
where the last inequality follows from a suitable choice of $R>0$. 
Therefore, it follows from \eqref{eq:int3proof} and \eqref{eq:E4_proof} that $\|\mathcal{T}_4^j\|^2$ is upper-bounded as
\begin{equation} \label{eq:T4bound}
    \|\mathcal{T}_4^j\|^2 = \sum_{l=1}^k\sum_{\substack{j'=1\\j'\neq j}}^k \|\bm E^4_{j',l} \|^4\|\bm q_l^t \|^2/n^2 \leq \frac{\pi_{\min}^{2(1+\zeta^{-1})}}{k} \sum_{l=1}^k \|\bm q_l^t \|^2.
\end{equation}
Combining \eqref{eq:T1_bound}, \eqref{eq:T3bound}, and \eqref{eq:T4bound} completes the proof.

%%%%%%%%%%%%%%%%%%%%%%%%%%%%%%%%%%%%%%%%%%%%%%%%%%%%%%%%%%%%%%%%%%%%%%%%%%%%%%%
%%%%%%%%%%%%%%%%%%%%%%%%%%%%%%%%%%%%%%%%%%%%%%%%%%%%%%%%%%%%%%%%%%%%%%%%%%%%%%%

\subsection{Minimax Lower Bound}
The following lemma is a simple generalization of the minimax lower bound by \citet[Proposition]{ghosh2021max} for max-affine regression to difference of max-affine regression. 
\begin{lemma}
\label{lemm:minimax}
Recall the difference of max-affine observation model written as
\begin{equation}
y = \max_{j \in [k_1]} \left( \bm \beta_j^\top [\bm x;1] \right) 
   - \max_{l \in [k_2]} \left( \bm  \alpha_l^\top [\bm x;1] \right) 
   + z,
\end{equation}
where $\bm x \in \mathbb{R}^d$ has i.i.d. Gaussian entries with mean zero and unit variance, 
$z$ is independent Gaussian noise with variance $\sigma_z^2$.
Let $\bm B = [\bm \beta_1, \ldots, \bm \beta_{k_1}]$ and $\bm A = [\bm \alpha_1, \ldots, \bm \alpha_{k_2}]$ 
collect the parameters as columns. With the dataset $\{(\bm x_i,y_i)\}_{i=1}^n$, 
denote the design matrix by $\bm \Xi \in \mathbb{R}^{n \times (d+1)}$ with rows $\bm \xi_i^\top = [\bm x_i;1]^\top$.  
Then we have that 
\begin{equation}
\label{eq:main_mini_x}
\inf_{\hat{\bm B}, \hat{\bm A}} ~ \sup_{\bm B, \bm A} 
\mathbb{E} \left[ \frac{1}{n} \left\| \bm \Xi(\hat{\bm B} - \bm B) \right\|_{\mathrm{F}}^2 
+ \frac{1}{n} \left\| \bm \Xi(\hat{\bm A} - \bm A) \right\|_{\mathrm{F}}^2 \right] \geq C \frac{\sigma_z^2 (k_1+k_2)d}{n}.
\end{equation}
\end{lemma}

\begin{proof}
We proceed with a standard application of Fano’s method.

Let the tolerance parameter be denoted by $\epsilon > 0$. 
Define the local class of design-transformed parameter matrices:
\begin{equation}
\mathcal{F} 
= \Big\{ \bm \Xi \bm B \in \mathbb{R}^{n \times k_1}, \; 
           \bm \Xi \bm A \in \mathbb{R}^{n \times k_2} :
           \| \bm \Xi \bm B \|_{\mathrm{F}} \leq  4\epsilon \sqrt{n k_1}, 
           \| \bm \Xi \bm A \|_{\mathrm{F}} \leq 4\epsilon \sqrt{n k_2}
   \Big\}.
\end{equation}
We construct a $2 \epsilon \sqrt{n}$-packing of each column space under the Frobenius norm as $\{\bm B^1, \ldots, \bm B^M, \bm A^1, \ldots, \bm A^M\}$. This can be achieved by a $\ell_2$-norm $2\epsilon\sqrt{n}$-packing of each column of $\bm B$ and $\bm A$.  
By standard volumetric packing arguments, there exist packings with cardinality $2M$ such that 
\begin{equation}
\log M \geq  C  (k_1+k_2)(d+1).
\end{equation}
For any two distinct packed elements $(\bm B^j, \bm A^i)$ and $(\bm B^j, \bm A^j)$, we have
\begin{equation}
2 \epsilon \sqrt{k_1 + k_2} \;\;\le\;\;
\frac{1}{\sqrt{n}} \left( 
   \| \bm \Xi(\bm B^i - \bm B^j) \|_{\mathrm{F}}^2
 + \| \bm \Xi(\bm A^i - \bm A^j) \|_{\mathrm{F}}^2 
\right)^{1/2}
\leq 8 \epsilon \sqrt{k_1 + k_2},
\end{equation}
where the left inequality follows from the column space packing definition, and the right inequality follows from the maximum $\ell_2$ separation of the individual column packings. 
Thus, the packing elements are well-separated in the prediction space. With a slight abuse of notation, let $\max \bm X$ denote the row-wise max-function on the matrix $\bm X$.
Next, let $P_{(\bm B,\bm A)}$ denote the distribution of the observations under parameters $(\bm B, \bm A)$, i.e. $P_{(\bm B,\bm A)} = \mathrm{Normal}\left(\max \bm \Xi \bm B - \max \bm \Xi A, \sigma_z^2\bm I_n \right)$.  
For two elements $(\bm B^i,\bm A^i)$ and $(\bm B^j,\bm A^j)$, the KL divergence is bounded by
\begin{align}
D_{\mathrm{KL}}\!\left( P_{(\bm B^i,\bm A^i)} ||  P_{(\bm B^j,\bm A^j)} \right)
\leq  \frac{1}{2\sigma_z^2}
\left\| 
   \left( \max_{j\in[k_1]} \bm \Xi \bm B^{i} - \max_{l\in[k_2]} \bm \Xi \bm A^i \right)
   - 
    \left( \max_{j\in[k_1]} \bm \Xi \bm B^j - \max_{l\in[k_2]} \bm \Xi \bm A^j \right)
\right\|_2^2. 
\end{align}
Using the fact that the row-wise max function is $1$-Lipschitz with respect to its inputs in $\ell_2$ norm, we obtain
\begin{equation}
D_{\mathrm{KL}}\!\left( P_{(\bm B^i,\bm A^i)} ||  P_{(\bm B^j,\bm A^j)} \right)\leq \frac{1}{2\sigma_z^2}
\left( \| \bm \Xi(\bm B^i - \bm B^j) \|_{\mathrm{F}}^2 + \| \bm \Xi(\bm A^i - \bm A^j) \|_{\mathrm{F}}^2 \right).
\end{equation}
From the packing construction, this yields the upper bound
\begin{equation}
D_{\mathrm{KL}}\!\left( P_{(\bm B^i,\bm A^i)} ||  P_{(\bm B^j,\bm A^j)} \right)\leq \frac{32 (k_1+k_2) \epsilon^2 n}{\sigma_z^2}.
\end{equation}
Finally, Fano's inequality requires that
\begin{equation}
\frac{1}{M^2} \sum_{i,j \in [M]} 
D_{\mathrm{KL}}\!\left( P_{(\bm B^i,\bm A^i)} ||  P_{(\bm B^j,\bm A^j)} \right) 
+ \log 2
\leq \frac{1}{2} \log M,
\end{equation}
which is satisfied provided
\begin{equation}
\epsilon^2 \leq C \frac{\sigma_z^2 (d+1)}{n}.
\end{equation}
Then Fano’s method \citep[Proposition 15.2]{wainwright2019high} yields \eqref{eq:main_mini_x}, which completes the proof.
\end{proof}

 Notice that since $\bm \Xi$ has i.i.d. Gaussian rows, by standard random matrix results, its singular values concentrate around $\sqrt{n}$. Hence, with high probability,
$
 \| \bm \Xi \bm \Delta \|_2 \approx \sqrt{n}$ for any $ \bm \Delta \in \mathbb R^{(d+1)\times k}$.
Thus, the prediction-norm lower bound converts into a parameter-norm lower bound of the same order as
\begin{equation}
\inf_{\hat{\bm B}, \hat{\bm A}} ~ \sup_{\bm B, \bm A} 
\mathbb{E} \left[   \left\|  (\hat{\bm B} - \bm B) \right\|_{\mathrm{F}}^2 
+   \left\|  (\hat{\bm A} - \bm A) \right\|_{\mathrm{F}}^2 \right] \geq C \frac{\sigma_z^2 (k_1+k_2)d}{n}.
\end{equation}

\subsection{Motivating the Smoothness of the Inverse Regression Curve with DoMA functions}\label{append:smoothness}
Recall that the inverse regression curve (IRC) is defined as $\mathbb E[\bm x| y]$. The goal to show that, although DoMA functions are inherently non-smooth, $m(y)$ can still be a smooth function of $y$. In what follows, we assume that $\bm x $ has a Lipschitz smooth sub-Gaussian distribution with bounded support $\mathcal{B}$ such that $\sup_{\bm x\in\mathcal{B}}\|\bm x\|_\infty \leq L$. For notational simplicity, we consider the max-affine model written as
\[
y  = \max_{j \in [k]}\bm \beta_j^\top [\bm x;1]+z,  
\]
and the following calculation easily extends to the DoMA case. The IRC can be written as
\begin{align}
\label{eq:smooth_decomposition}
     \mathbb{E}[\bm x|y] &= \sum_{j=1}^k \mathbb{P}(\bm x\in \mathcal{C}_j(\bm \beta)) \mathbb E[\bm x| y , ~x\in \mathcal{C}_j(\bm \beta)] + \mathbb P(\bm x\in \mathcal{V}(\bm \beta)) \mathbb E[\bm x| y ,~ x\in \mathcal{V}(\bm \beta)] \nonumber\\ 
    & = \sum_{j=1}^k \mathbb{P}(\bm x\in \mathcal{C}_j(\bm \beta)) \mathbb E[\bm x| y -[\bm \beta_j]_{d+1}= [\bm \beta_j]_{1:d}^\top \bm x +z,~ x\in \mathcal{C}_j(\bm \beta)] \,,
\end{align}
where the last equality follows from $\mathbb{P}(\bm x\in \mathcal{V}(\bm \beta))=0$ since $\mathcal{V}(\bm \beta)$ is a degenerate set and $\bm x$ has a Lipschitz smooth distribution, and the observation that 
\[
\|\mathbb E[\bm x| y , x\in \mathcal{V}(\bm \beta)]\|_\infty \leq \mathbb E[\|\bm x\|_\infty| y , x\in \mathcal{V}(\bm \beta)] \leq \sup_{\bm x\in \mathcal{B}}\|\bm x\|_\infty \leq L < \infty.
\]
We have now shown in \eqref{eq:smooth_decomposition} that the IRC of a max-affine model can be expressed as a convex combination of the IRCs of linear models. The IRC of linear models with sub-Gaussian distributions is known to satisfy the required smoothness properties to invoke the guarantees by the sliced inverse regression method \citep[Assumption 1]{tan2018convex}. The authors also specifically discuss the linear model in Example 1 of their numerical studies section. Overall, we expect max-affine models (and DoMA models by extension) to have smooth IRCs. This makes our DoMA model a good candidate for dimension reduction methods requiring such conditions such as sliced inverse regression and outer product of gradients (OPG) \cite{xia2002adaptive}. However, the guarantees of such methods only apply in the asymptotic of the sample size and require careful tuning of additional hyperparameters that crucially impact the methods' performance. For these reasons and other technical considerations, the spectral method by \citet{ghosh2021max} is adopted for initialization purposes in the main text.

\subsection{Proof of Lemma \ref{lemma:M1M2}}\label{append:proof:m1m2}
Let $\mathcal{K}=[k_1]\times[k_2]$. Then we can rewrite the regression model in \eqref{eq:def_model_xi} as
\begin{equation*}
    y = \sum_{(j,l)\in \mathcal{K}} \mathbf{1}_{\{\bm x \in \mathcal{C}_j(\bm \beta)\cap \mathcal{C}_l(\bm \alpha)\}} \langle \bm \beta_j -\bm \alpha_l, \bm \xi \rangle+z.
\end{equation*}
%In what follows we will drop the subscript $i$ for simplicity.
Then the expression of $\bm m_1$ is equivalently rewritten as
\begin{align*}
    \bm m_1 &= \mathbb E\left[y\bm x\right]\nonumber \\&= \sum_{(j,l)\in \mathcal{K}} \mathbb E\left[ \left(\mathbf{1}_{\{\bm x \in \mathcal{C}_j(\bm \beta)\cap \mathcal{C}_l(\bm \alpha)\}}\langle \bm \beta_j -\bm \alpha_l, \bm \xi\rangle+z\right)\bm x\right] = \sum_{(j,l)\in \mathcal{K}} \mathbb E\left[ \mathbf{1}_{\{\bm x \in \mathcal{C}_j(\bm \beta)\cap \mathcal{C}_l(\bm \alpha)\}}\langle \bm \beta_j -\bm \alpha_l, \bm \xi_i\rangle\bm x\right] \nonumber \\
    &= \sum_{(j,l)\in \mathcal{K}} \mathbb E\left[ \mathbf{1}_{\{\bm x \in \mathcal{C}_j(\bm \beta)\cap \mathcal{C}_l(\bm \alpha)\}}\right]\left[\bm \beta_j -\bm \alpha_l\right]_{1:d} = \sum_{(j,l)\in \mathcal{K}} \mathbb P\left( \bm x \in \mathcal{C}_j(\bm \beta)\cap \mathcal{C}_l(\bm \alpha)\right)\left[\bm \beta_j -\bm \alpha_l\right]_{1:d},
\end{align*}
 where the second inequality follows from the independence of $\bm x$ and $\bm z$, and the third equality follows from Stein's Lemma for Gaussian vectors, i.e. $\mathbb{E}\left[g(\bm x) \bm x\right] =\mathbb{E}\left[\nabla g(\bm x)\right]$.
Similarly, the expectation in $\bm M_2$ is calculated as
 \begin{align*}
     \bm M_2 &= \mathbb E\Bigg[\left(\sum_{(j,l)\in \mathcal{K}} \mathbf{1}_{\{\bm x_i \in \mathcal{C}_j(\bm \beta)\cap \mathcal{C}_l(\bm \alpha)\}}\langle \bm \beta_j -\bm \alpha_l, \bm \xi_i\rangle+z_i \right)\left(\bm x\bm x^\mathsf{T}-\bm I_d\right)\Bigg]\nonumber\\
     &= \sum_{p=1}^d\mathbb E\Bigg[\sum_{(j,l)\in \mathcal{K}} \mathbf{1}_{\{\bm x_i \in \mathcal{C}_j(\bm \beta)\cap \mathcal{C}_l(\bm \alpha)\}}\langle \bm \beta_j -\bm \alpha_l, \bm \xi_i\rangle\left(\bm x[\bm x]_p-\bm e_p\right)\Bigg]\bm e_p^\mathsf{T}\nonumber\\
     &= \sum_{p=1}^d\mathbb E\Bigg[\sum_{(j,l)\in \mathcal{K}} \mathbf{1}_{\{\bm x_i \in \mathcal{C}_j(\bm \beta)\cap \mathcal{C}_l(\bm \alpha)\}}\bigg([\bm \beta_j-\bm \alpha_l]^\mathsf{T}_{1:d}\bm x [\bm x]_p\bm x+[\bm \beta_j-\bm \alpha_l]_{d+1}[\bm x]_p\bm x\nonumber\\
     &\quad\hspace{1.3cm}-[\bm \beta_j-\bm \alpha_l]_{1:d}\bm x \bm e_p - [\bm \beta_j-\bm \alpha_l]_{d+1}\bm e_p\bigg)\Bigg]\bm e_p^\mathsf{T}\nonumber\\
     & = \sum_{p=1}^d \sum_{(j,l)\in \mathcal{K}}\Bigg[ \bigg([\bm \beta_j-\bm \alpha_l]_{1:d}\bm e_p^\mathsf{T} + \bm e_p [\bm \beta_j-\bm \alpha_l]_{1:d}^\mathsf{T}\bigg)\mathbb{E}\bigg[\mathbf{1}_{\{\bm x_i \in \mathcal{C}_j(\bm \beta)\cap \mathcal{C}_l(\bm \alpha)\}}\bm x\bigg] \nonumber \\
     &\quad\hspace{2cm} +[\bm \beta_j-\bm \alpha_l]_{d+1}\bm e_p \mathbb{P}\left(\bm x_i \in \mathcal{C}_j(\bm \beta)\cap \mathcal{C}_l(\bm \alpha)\right) - \bm e_p [\bm \beta_j-\bm \alpha_l]_{1:d}^\mathsf{T}\mathbb{E}\left[\mathbf{1}_{\{\bm x_i \in \mathcal{C}_j(\bm \beta)\cap \mathcal{C}_l(\bm \alpha)\}}\bm x\right] \nonumber\\ &\quad\hspace{2cm}-[\bm \beta_j-\bm \alpha_l]_{d+1}\bm e_p \mathbb{P}\left(\bm x_i \in \mathcal{C}_j(\bm \beta)\cap \mathcal{C}_l(\bm \alpha)\right)\Bigg]\bm e_p^\mathsf{T} \nonumber \\
     &= \sum_{p=1}^d \sum_{(j,l)\in \mathcal{K}} [\bm \beta_j-\bm \alpha_l]_{1:d} \bm e_p^\top \mathbb{E}\left[\mathbf{1}_{\{\bm x_i \in \mathcal{C}_j(\bm \beta)\cap \mathcal{C}_l(\bm \alpha)\}}\bm x\right] \bm e_p^\top =  \sum_{(j,l)\in \mathcal{K}} [\bm \beta_j-\bm \alpha_l]_{1:d}  \mathbb{E}\left[\mathbf{1}_{\{\bm x_i \in \mathcal{C}_j(\bm \beta)\cap \mathcal{C}_l(\bm \alpha)\}}\bm x\right]^\top,
 \end{align*} 
where the second equality follows from the independence of $\bm x$ and $\bm z$, and the third equality follows from Stein's Lemma.
\subsection{Proof of Lemma \ref{lemm:Span}}\label{append:proof:span}
We can assume without loss of generality that $\bm v = \bm\alpha_1$. Let $\oplus$ denote the Minkowski set addition. Then we have that
\begin{align}
         \label{eq:span_proof2}
         \mathrm{span}\left(\{\bm \beta_j\}_{j=1}^{k_1}\ominus\{\bm \alpha_l\}_{l=1}^{k_2}\right) &=\mathrm{span}\left(\{\bm \beta_j-\bm v\}_{j=1}^{k_1}\ominus\{\bm \alpha_l-\bm v\}_{l=1}^{k_2}\right) \nonumber \\
         &= \mathrm{span}\left(\{\bm \beta_j-\bm v\}_{j=1}^{k_1}\ominus\left(\bm 0\cup\{\bm \alpha_l-\bm v\}_{l=2}^{k_2}\right)\right) \nonumber \\
          &= \mathrm{span}\left(\{\bm \beta_j-\bm v\}_{j=1}^{k_1}\right) \oplus  \mathrm{span}\left(\{\bm \beta_j-\bm v\}_{j=1}^{k_1}\ominus\{\bm \alpha_l-\bm v\}_{l=2}^{k_2}\right)\nonumber\\
          & = \mathrm{span}\left(\{\bm \beta_j-\bm v\}_{j=1}^{k_1}\right)\oplus\mathrm{span}\left(\{\bm \alpha_l-\bm v\}_{l=2}^{k_2}\right),
\end{align}
which yields the assertion in \eqref{eq:span_statement}. The number of vectors used in the last line of \eqref{eq:span_proof2} yields the bound in \eqref{eq:rank_statement}. 
%%%%%%%%%%%%%%%%%%%%%%%%%%%%%%%%%%%%%%%%%%%%%%%%%%%%%%%%%%%%

%%%%%%%%%%%%%%%%%%%%%%%%%%%%%%%%%%%%%%%%%%%%

\end{document}